\RequirePackage{latexml}
\iflatexml
  \documentclass{amsart}
  \newcommand{\Description}[2][]{}
  \newtheorem{theorem}{Theorem}
\else
  \documentclass[acmsmall, screen,nonacm]{acmart}
  \authorsaddresses{}
  \setcopyright{none}
  \renewcommand\footnotetextcopyrightpermission[1]{}
\fi

\usepackage{subcaption}
\usepackage{amsmath}
\usepackage{xspace}
\usepackage{multicol}
\usepackage{multirow}
\usepackage{booktabs}
\usepackage{cleveref}
\usepackage[super]{nth}
\usepackage{enumitem}
\usepackage{amsmath}
\usepackage{makecell}
\usepackage{textcomp}
\usepackage{textgreek}
\usepackage{pifont}
\usepackage{wrapfig}
\usepackage{listings}
\lstdefinestyle{promptblock}{
    language={},
    basicstyle=\fontfamily{pcr}\selectfont\scriptsize,
    keywordstyle=\fontfamily{pcr}\selectfont\scriptsize,
    commentstyle=\fontfamily{pcr}\selectfont\scriptsize,
    stringstyle=\fontfamily{pcr}\selectfont\scriptsize,
    identifierstyle=\fontfamily{pcr}\selectfont\scriptsize,
    breaklines=true,
    breakatwhitespace=false,
    columns=fullflexible,
    frame=single,
    framesep=3pt,
    framerule=0.3pt,
    rulecolor=\color{black!35},
    backgroundcolor=\color{black!3},
    keepspaces=true,
    showstringspaces=false,
    aboveskip=0.4em,
    belowskip=0.6em,
    xleftmargin=0.5em,
    xrightmargin=0.5em
}

\usepackage{siunitx}
\DeclareSIUnit\dollar{\$}
\DeclareSIUnit{\month}{month}
\DeclareSIUnit{\thousand}{k}
\DeclareSIUnit{\million}{M}

\newcommand{\Wh}[1]{\SI{#1}{{\watt\hour}}}

\usepackage{array}
\newcolumntype{L}[1]{>{\raggedright\let\newline\\\arraybackslash\hspace{0pt}}m{#1}}
\newcolumntype{C}[1]{>{\centering\let\newline\\\arraybackslash\hspace{0pt}}m{#1}}
\newcolumntype{R}[1]{>{\raggedleft\let\newline\\\arraybackslash\hspace{0pt}}m{#1}}

\usepackage{pifont}
\newcommand{\circled}[1]{\ding{\numexpr#1+201}}
\newcommand{\circledtext}[1]{\raisebox{.5pt}{\textcircled{\raisebox{-0.9pt} {\small #1}}}}

\newtheorem*{theorem*}{Theorem}
\theoremstyle{definition}
\newtheorem{assumption}{Assumption}

\usepackage[binary-units=true]{siunitx}
\newcommand{\SIx}[1]{\num{#1}\relax}
\DeclareSIUnit\per{/}
\DeclareSIUnit\dollar{\$}
\DeclareSIUnit{\month}{month}
\DeclareSIUnit{\thousand}{k}
\DeclareSIUnit{\million}{M}

\usepackage{pgfplots}
\usepgfplotslibrary{fillbetween}
\usepackage{tikz}
\usetikzlibrary{matrix,positioning}
\usetikzlibrary{external}
\usepackage{pgfplots}
\usepackage{pgfplotstable}
\usetikzlibrary{pgfplots.groupplots}
\usetikzlibrary{arrows}
\usetikzlibrary{patterns}
\usetikzlibrary{positioning}
\usetikzlibrary{decorations.pathreplacing}
\usetikzlibrary{shapes.arrows}
\usetikzlibrary{shapes.geometric,shapes.misc}
\usetikzlibrary{pgfplots.groupplots}
\pgfplotsset{compat=newest}
\pgfkeys{/pgf/number format/.cd,1000 sep={}}

\pgfplotsset{
    discard if/.style 2 args={
        x filter/.code={
            \edef\tempa{\thisrow{#1}}
            \edef\tempb{#2}
            \ifx\tempa\tempb
                
            \fi
        }
    },
    discard if not/.style 2 args={
        x filter/.code={
            \edef\tempa{\thisrow{#1}}
            \edef\tempb{#2}
            \ifx\tempa\tempb
            \else
                
            \fi
        }
    }
}

\makeatletter
\newcommand\resetstackedplots{%
\pgfplots@stacked@isfirstplottrue
}
\makeatother

\usetikzlibrary{spy}
\usetikzlibrary{patterns,backgrounds}
\pgfdeclarelayer{foreground}
\pgfsetlayers{background,main,foreground}

\usepgfplotslibrary{statistics}

\usepackage[skins]{tcolorbox}

\colorlet{takeawaybg}{black!5!white}
\colorlet{takeawayframe}{black!25!white}

\newcounter{takeawaycounter}
\crefname{takeawaycounter}{Takeaway}{Takeaways}
\Crefname{takeawaycounter}{Takeaway}{Takeaways}

\newlength{\takeawaypad}
\newsavebox{\takeawayboxbox}

\newenvironment{takeawaybox}[2]{%
  \refstepcounter{takeawaycounter}\label{#1}%
  \par\addvspace{0.6\baselineskip}\noindent
  \begingroup
  \setlength{\fboxrule}{1pt}%
  \setlength{\fboxsep}{\takeawaypad}%
  \begin{lrbox}{\takeawayboxbox}%
  \begin{minipage}{\dimexpr\linewidth-2\fboxsep-2\fboxrule\relax}
  \textbf{Takeaway~\arabic{takeawaycounter}:~#2}\par\addvspace{0.25\baselineskip}%
}{%
  \end{minipage}%
  \end{lrbox}%
  \fcolorbox{takeawayframe}{takeawaybg}{\usebox{\takeawayboxbox}}%
  \endgroup
  \par\addvspace{0.6\baselineskip}%
}

\usepackage{soul}
\makeatletter
\def\SOUL@hlpreamble{%
\setul{}{2.2ex}
\let\SOUL@stcolor\SOUL@hlcolor
\SOUL@stpreamble
}
\makeatother
\soulregister\name7
\soulregister\cite7
\soulregister\ref7
\soulregister\Cref7
\soulregister\pageref7
\soulregister\eg7
\soulregister\ie7
\soulregister\texttt7
\soulregister\circled7
\soulregister\encircle7
\soulregister\uline7
\soulregister\circled7
\soulregister\circledtext7
\soulregister\todo7
\soulregister\henry1

\newcommand{\name}{AgentDecarbonizer\xspace}

\newcommand{\eg}{e.g.,\xspace}
\newcommand{\ie}{i.e.,\xspace}

\newcommand{\unitCI}{gCO\textsubscript{2}e/kWh\xspace}
\newcommand{\carbonKG}{kgCO\textsubscript{2}e\xspace}
\newcommand{\carbonG}{gCO\textsubscript{2}e\xspace}

\begin{document}

\title{\name{}: Carbon-Aware Execution for AI~Agents}

\iflatexml
  \author{Leyi Yan}
  \address{University of Waterloo}
  \author{Shuangning Li}
  \address{University of Chicago}
  \author{Sihang Liu}
  \address{University of Waterloo}
\else
  \author{Leyi Yan}
  \affiliation{%
    \institution{University of Waterloo}
    \city{Waterloo}
    \state{Ontario}
    \country{Canada}
  }

  \author{Shuangning Li}
  \affiliation{%
    \institution{University of Chicago}
    \city{Chicago}
    \state{IL}
    \country{USA}
  }

  \author{Sihang Liu}
  \affiliation{%
    \institution{University of Waterloo}
    \city{Waterloo}
    \state{Ontario}
    \country{Canada}
  }
\fi

\begin{abstract}
AI agents extend large language models from single prompt-response interactions to long-running, goal-directed workflows that issue many model calls, invoke tools, and interact with external environments.
These workflows enable tasks such as software repair, data analysis, and experiment management, but their repeated model invocations can incur substantial carbon emissions.
This paper characterizes the carbon emissions of OpenClaw agent workloads using WildClawBench, and shows that emissions depend on token consumption, context cache reuse, and the carbon intensity of the grid.
Our characterization identifies deadline flexibility as an opportunity for carbon-aware execution: agent tasks can wait for lower-carbon-intensity periods or shift to lower-carbon grids.
However, doing so requires handling uncertain execution time for temporal shifting and cached context recomputation during spatial shifting.
We present \name{}, a carbon optimizer for AI agents that runs alongside OpenClaw.
Given a task prompt and user-specified deadline, \name{} conservatively estimates task duration and selects deadline-feasible execution schedules, while accounting for cache recomputation overhead during spatial shifting.
Evaluated on WildClawBench workloads with 60 agent tasks across four grids, \name{} reduces carbon emissions by up to 57.9\,\% compared with a carbon-agnostic baseline and by up to 37.5\,\% compared with a baseline that selects the carbon-optimal grid at task start time.
\end{abstract}

\maketitle

\section{Introduction}

AI agents extend large language models (LLMs) from single prompt-response interactions into autonomous, goal-directed workflows.
Given a high-level task, an agent plans intermediate steps, invokes model APIs and external tools, and refines its execution based on intermediate observations~\cite{yao2023react,toolformer,wang2024survey}.
This execution model enables long-running tasks such as software repair, data analysis, and machine learning experimentation to proceed with limited user intervention~\cite{sweagent,huang2024mlagentbench,data_interpreter}. 
For example, OpenAI recently introduced scheduled tasks which have flexible deadlines~\cite{openaiChatGPTScheduledTasks}. 
Unlike traditional LLM queries, an agent task may expand into many model calls, tool invocations, and environmental interactions.
These agent workloads have been supported by dedicated  platforms~\cite{googleAntigravity,openaiCodex,anthropicClaudeCode} and OpenClaw is a representative open-source example~\cite{openclawintro}.
Agent execution also differs from conventional LLM workloads through repeated reuse of context across model invocations~\cite{yuan2026agenticaiworkloadcharacteristics,webarena,visualwebarena,mckinsey2025stateofai}.
Because agents maintain states such as task goals, conversation history, and tool outputs throughout execution, modern LLM serving systems use context caching to avoid recomputing shared context and reduce inference overhead~\cite{openaiPromptCaching,cachedattention,anthropicPromptCaching,googleGeminiContextCaching,awsbedrockPromptCache}.

While autonomous agents enable complex task processing, their multi-step execution introduces substantial computational demands.
To characterize these workloads, we analyze OpenClaw using WildClawBench~\cite{ding2026wildclawbench}, a benchmark suite of 60 realistic agent tasks across six categories, such as productivity and coding. 
We observe that agent execution is highly diverse, exhibiting significant variations in execution time, API calls, and token consumption.
Moreover, these workloads are input-heavy, with input tokens accounting for 99.1\% of all processed tokens. 
This makes context caching highly critical, as cached input tokens eliminate recomputation for 94.6\% of input tokens.

Their computation-intensive workload characteristics also have sustainability implications. 
Serving large-scale AI models leads to high carbon emissions~\cite{elsworth2025environmentalimpact,ieaenergyai2025,faiz2024llmcarbon}. 
The long-running and intensive execution of autonomous agents further amplifies this impact.
We model the energy consumption of AI agents based on prior studies \cite{jegham2025howhungryai,altman2025gentlesingularity,elsworth2025environmentalimpact} and vendor API pricing \cite{openaiPricing,googleGeminiPricing}, extending existing per-prompt energy measurements into per-token estimates for input, output, and cached input tokens. 
We then convert the estimated energy consumption into carbon emissions using the carbon intensity of the grid (in units of ${\rm \unitCI{}}$)---a factor that varies across locations and over time with the energy generation mix~\cite{souza2024casper,carboncast}.
For example, executing WildClawBench on the CISO grid (California, USA) on Feb 24, 2026 is estimated to emit 6.6~\carbonKG{}, equivalent to driving a passenger vehicle for approximately \SI{25}{\kilo\meter}~\cite{epaTypicalPassengerVehicleGHG}. 
These emissions can become substantial when agents are deployed at scale, motivating carbon emission optimizations for AI agents.


Because carbon intensity varies across time and locations, existing carbon-aware optimization techniques reduce emissions by shifting workloads to lower-carbon periods and regions~\cite{radovanovic2022carbon,guo2023carbon,sukprasert2023limitations,xu2020managing}.
However, these techniques do not directly address the unique execution characteristics of AI agents.
First, agent execution duration is uncertain because later steps may depend on earlier model outputs, tool results, and interactions with external environments. 
Meanwhile, user-specified deadlines must be met rather than indefinitely deferring execution for lower-carbon periods.
Second, reusable context caches are critical for agent execution.
In our characterization, context caching reduces carbon emissions by 80.7\,\%.
Maintaining agent execution in the same location can preserve these cached contexts, which is usually assumed by serving systems \cite{moocacke2024arxiv,googleGeminiContextCachingUse,awsbedrockPromptCache}.
In contrast, shifting execution to another grid can lose cache and require recomputing previously cached inputs, adding extra emissions.
As a result, spatial shifting is not always beneficial: the grid with the lowest carbon intensity may not yield the lowest total emissions once the cache recomputation overhead is considered.

To address these practical challenges, we present \name{}, a carbon optimization tool for autonomous AI agents on OpenClaw.
Given a task prompt and user-specified deadline, \name{} runs alongside OpenClaw to estimate task duration, forecast carbon intensity, and select deadline-feasible execution schedules across time and grids.
To handle uncertain agent execution time, we design an execution time estimator based on a lightweight Gemma~4 4B model~\cite{gemma4qat} that runs on the local device.
It begins with a conservative duration estimate based on the task phases and their complexity, and refines the remaining-duration estimate during execution using the OpenClaw agent's progress, API-call statistics, and observed tool usage.
This allows the agent task to wait for lower-carbon periods when there is enough time and to fall back to immediate execution when the deadline is tight.

To account for the impact of caching during spatial workload shifting, \name{} estimates the recomputation overhead using model token statistics.
It performs grid shifting only when the expected carbon reduction exceeds this cache overhead.
We formulate this cache-aware optimization problem under a deadline constraint and derive a dynamic-programming solution based on the observation that future cache costs depend solely on the previous execution location.

We evaluate \name{} on top of OpenClaw \cite{openclawintro}, a representative open-source agent system, using a lightweight Mac mini to run the local OpenClaw client.
The agent invokes GPT-5.4 as the remote model API, and we run 60 WildClawBench tasks spanning six categories.
While this work targets OpenClaw, the optimization strategies are broadly applicable to other agent systems.
We construct six long-running workloads by composing tasks within the same category. 
The evaluation uses carbon intensity traces from four grids and workload deadlines ranging from 3 hours to 24 hours. 
We compare \name{} against two baselines: a carbon-agnostic baseline that uses the average emissions across all four grids and a current-optimal baseline that selects the grid with the lowest carbon intensity at the workload start time.
The evaluation shows that, with a flexible 24-hour deadline, \name{} reduces carbon emissions by 57.9\,\% over the average baseline and 37.5\,\% over the current-optimal baseline.
Even under a tighter 3-hour deadline, \name{} reduces carbon emissions by 34.9\,\% over the average baseline and 3.3\,\% over the current-optimal baseline.

This paper makes the following contributions:
\begin{itemize}[leftmargin=1em]
    \item To the best of our knowledge, this is the first work to systematically study the carbon emissions of AI agents and optimize carbon using methods tailored to their execution characteristics. 
    
    \item We systematically characterize the carbon emissions of AI agent workloads on OpenClaw, and identify the key factors that affect their emissions. 

    \item We design \name{}, a carbon optimizer for AI agents that estimates agent duration
    and progress, forecasts regional carbon intensity, and selects deadline-feasible low-carbon execution plans with cache-aware grid shifting decisions.


    \item We evaluate \name{} using AI agent workloads from WildClawBench and demonstrate that it always meets workload deadlines while achieving up to 57.9\,\% carbon emission savings compared to a carbon-agnostic baseline under a flexible 24-hour deadline. 
\end{itemize}
\section{Background and Motivation}
\label{sec:background}

In this section, we introduce autonomous AI agents, OpenClaw, and carbon emission accounting.

\subsection{Large Language Models} \label{subsec:llm}

Recent advances in large language models (LLMs) have demonstrated powerful capabilities in natural language understanding, logical reasoning, and code generation.
LLM inference typically consists of two execution phases: \textit{prefill} and \textit{decode}.
During the \textit{prefill} phase, the model processes the entire user prompt and context in a single, parallel forward pass. 
Subsequently, during the \textit{decode phase}, the model generates output tokens autoregressively.
Because of the different execution patterns, the prefill phase is \textit{compute-bound} while the decode phase is \textit{memory-bound}. 
There have been various system-level optimizations that improve the efficiency of AI model serving \cite{pagedattention,orca,agrawal2024sarathi,distserve,patel2024splitwise}.
Notably, one class of techniques employs context caching, which reuses precomputed Key-Value (KV) states for shared prompts across requests, thereby avoiding repeated prefill computation.
This allows future prompts or multi-turn queries with shared or similar context to reuse cached computation~\cite{yu2025smartcache,yan2025contextcache,cachedattention}. 
As LLMs become more capable and widely used, serving them requires substantial computational resources.
As a result, deploying these advanced models is particularly compute-intensive, requiring high-end GPUs or AI accelerators. 

\subsection{Autonomous AI Agents}


Advances in LLMs have enabled autonomous AI agents, which extend conventional chatbot-style prompt-response interactions into goal-directed workflows.
Instead of requiring the user to repeatedly issue instructions and interpret each response, an agent begins with an initial request that specifies the goal, requirements, and expected outcome.
It can plan actions, invoke tools, interact with external environments, validate intermediate results, and retry failed steps before producing the final outcome.
As a result, the user can wait for the completed result without monitoring intermediate steps, which provides greater flexibility in execution scheduling~\cite{openaiCodex,swebench,webarena,ding2026wildclawbench}.
For example, OpenAI's Scheduled Tasks feature allows agent workflows to be executed with flexible schedules and deadlines~\cite{openaiChatGPTScheduledTasks}.
Representative tasks include code repository analysis, large-scale document comprehension, data analysis, and software development.
This shift from user-guided, prompt-by-prompt interactions to goal-directed workflows in realistic task environments is also reflected in recent AI agent studies~\cite{webarena,swebench,agentbench,ding2026wildclawbench,mckinsey2025stateofai,yuan2026agenticaiworkloadcharacteristics}.

\begin{figure}[t]
    \centering
    \includegraphics[width=\linewidth]{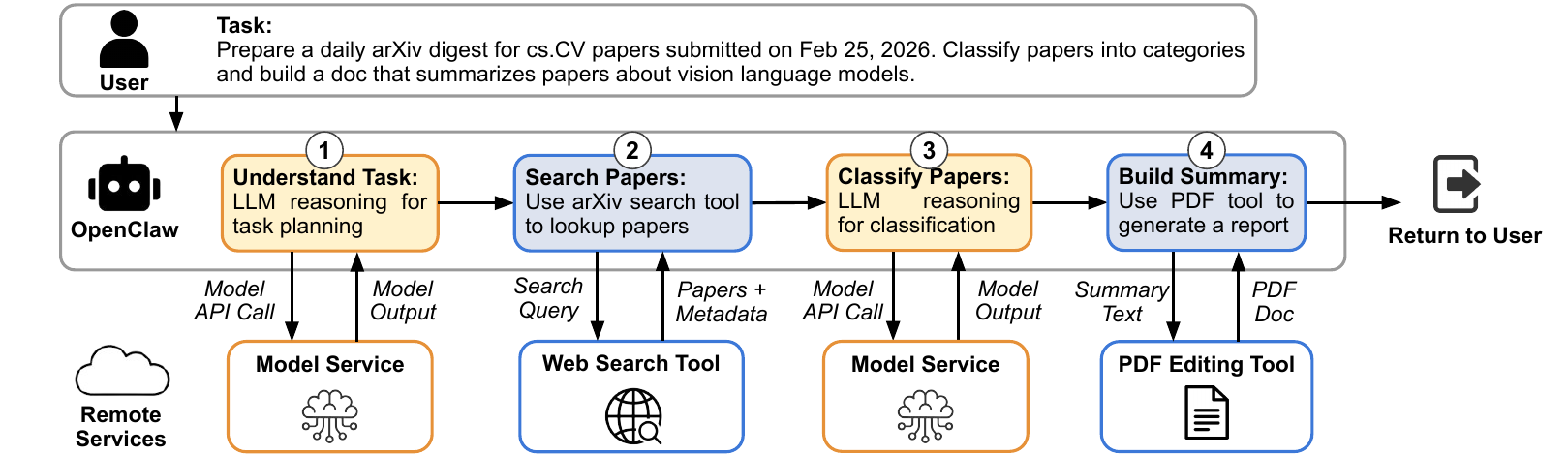}
    \caption{Example of an autonomous AI agent workflow on the OpenClaw platform~\cite{openclawintro}.}
    \Description{}
    \label{fig:agent-workflow}
\end{figure}

\subsection{OpenClaw} \label{subsec:openclaw}

Recently, autonomous agent capabilities have been integrated into systems from major AI service providers, such as Codex~\cite{openaiCodex}, Claude Code~\cite{anthropicClaudeCode}, and Antigravity~\cite{googleAntigravity}.
While these proprietary systems primarily target coding and software development, OpenClaw has emerged as a popular open-source autonomous agent system designed for more general tasks~\cite{openclawintro}.
OpenClaw serves as a workload harness that connects user requests, agent actions, external tools, and backend model APIs.
Through provider-specific adapters, OpenClaw can invoke different model APIs, such as GPT and Gemini, while keeping the agent runtime largely independent of the underlying model service~\cite{openclawModelAgnostic}.
These model APIs can also return request-level metadata, such as input, output, and cached token counts, alongside the generated response~\cite{openaiResponsesUsage}.
Such metadata provides visibility into agent execution behavior and resource usage.
\Cref{fig:agent-workflow} illustrates an OpenClaw workflow adapted from one task in WildClawBench~\cite{ding2026wildclawbench}.
In this example, the task is to summarize papers related to a specific topic.
OpenClaw coordinates planning, remote model reasoning, and tool-based data retrieval and processing before returning the final result.

\subsection{Carbon Accounting in Computing}
\label{subsec:carbon-account-computing}

Carbon emissions of LLM-based systems have already attracted attention from both academia and industry~\cite{faiz2024llmcarbon,Li2025fairpractical,elsworth2025environmentalimpact,ieaenergyai2025}.
Autonomous agent platforms such as OpenClaw can further amplify this impact. Their workflows often involve repeated model invocations and tool calls, increasing token consumption and execution time.
This motivates carbon-aware analysis for agent workloads.

In this section, we introduce the basic carbon accounting background for computing systems.
\Cref{sec:agent_carbon} then focuses on carbon accounting for AI agent systems.
Carbon accounting estimates the greenhouse gas emissions associated with completing a computing task.
Emissions from different sources are commonly translated into a shared unit, CO\textsubscript{2} equivalent (\carbonG{}), which allows their climate impact to be compared consistently.
In computing systems, carbon emissions are often divided into operational and embodied components.
Operational carbon captures the runtime emissions caused by electricity consumed during workload execution:
\begin{equation}
    C_{o} = E \times {\rm CI},
    \label{eq:operational-carbon}
\end{equation}
where $E$ denotes the workload energy consumption and ${\rm CI}$ denotes the average carbon intensity of the electricity supplied by the grid. When expressed in \unit{kWh} and \unitCI{}, respectively, $C_o$ is measured in \carbonG{}.
This model is commonly used for carbon accounting in computing systems~\cite{ACT,faiz2024llmcarbon,greenLLM,Li2025fairpractical}.

Computing systems also incur embodied carbon from the manufacturing processes of their components, such as CPUs, GPUs, memory, and storage.
Lifecycle accounting frameworks often amortize these emissions over the hardware lifetime and allocate a fraction to each workload based on execution time or resource usage~\cite{ACT}.
Embodied carbon is therefore important for full lifecycle assessments of AI infrastructure.
This paper targets carbon optimization through runtime execution decisions, such as when and where remote model API calls are executed.
These decisions directly affect operational carbon through workload energy consumption and grid carbon intensity, but they do not directly change the manufacturing emissions of the underlying hardware. 
Therefore, this work focuses on optimizing operational carbon.

\begin{figure}[t]
    \centering
    \begin{minipage}[t]{0.56\linewidth}
        \vspace*{0pt}
\captionof{table}{Average carbon intensity (CI, \unitCI{}) and major energy sources of different grids in 2025 \cite{ele_maps}.} \label{tab:grid_ci}
\small
\centering
\setlength{\tabcolsep}{3pt}
\begin{tabular}{llcl}
\toprule
Grid & Region & Avg CI & Major Energy Sources \\
\midrule
AT & Austria & 164 & Hydro, wind, gas, solar \\
CISO & California, USA & 208 & Solar, gas, wind, hydro \\
FI & Finland & 68 & Nuclear, wind, hydro \\
GB & Great Britain & 174 & Wind, gas, nuclear \\

\bottomrule
\end{tabular}


    \end{minipage}
    \hfill
    \begin{minipage}[t]{0.4\linewidth}
        \vspace*{0pt}
        \centering
        \includegraphics[width=\linewidth]{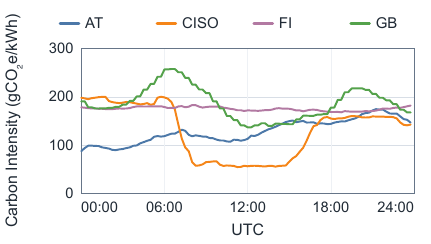}
        \caption{Carbon intensity on Feb 24, 2026.}
        \label{fig:ci_traces}
        \Description{}
    \end{minipage}
\end{figure}

A key factor of operational carbon is carbon intensity, which links energy use to carbon emissions.
\Cref{tab:grid_ci} lists four grids and their average carbon intensity in 2025 using data from Electricity Maps~\cite{ele_maps}.
Carbon intensity varies across regions because electricity grids rely on different energy generation mixes. Grids with larger shares of renewable energy, such as wind, solar, and hydro, generally have lower carbon intensity than grids that depend heavily on fossil fuels. Carbon intensity also varies over time within the same region as the generation mix changes. For example, increased solar generation during the day reduces carbon intensity.
\Cref{fig:ci_traces} shows the dynamics of carbon intensity for each grid in a day.
Because of the dynamics of carbon intensity, carbon-aware systems can reduce emissions by shifting computation across time or regions~\cite{souza2024casper,Caribou,zhang2026carbonawarecompute,gaia,CarbonExplorer,li2025ecoserve}.

\section{Carbon Emissions of AI Agents} \label{sec:agent_carbon}

In this section, we describe a carbon accounting method for AI agents and then provide estimates of per-token carbon emissions.

\subsection{Carbon Accounting Method for AI Agents}
\label{subsec:agent_carbon_accounting}

As introduced in \Cref{subsec:openclaw}, AI agent systems like OpenClaw split execution between a local user device and remote model APIs.
The local device provides the user interface and runs lightweight control logic, while computationally intensive model inference is delegated to the remote API service.
This execution model introduces two sources of operational carbon: local execution on the user device and remote execution by the model API service.
For the workloads considered in this paper, the local component is lightweight compared with remote LLM inference~\cite{sfailabsOpenClawHardware} and largely independent of the remote model API.
Therefore, local operational carbon acts mainly as a common background component, while remote API operational carbon is the major contributor to carbon emissions and therefore the primary optimization target.

Following the accounting scope defined in \Cref{subsec:carbon-account-computing}, we focus on the operational carbon $C$ of the remote model service.
We refer to each time block of length $\Delta t$ as an interval, and use $t$ to denote its start time.
For a task whose remote API calls may be served by different grids over its execution period, we calculate
\begin{equation}
    C = \sum_r \sum_t E_{r,t} \times CI_{r,t},
    \label{eq:time-grid-carbon}
\end{equation}
where $r$ refers to the grid, $E_{r,t}$ is the energy consumed by remote API execution served by grid $r$ during interval $[t,t+\Delta t)$, and $CI_{r,t}$ is the corresponding average carbon intensity.
This time- and grid-dependent accounting method captures the operational carbon associated with serving the remote model of an AI agent task.
This accounting formulation forms the basis of our optimization problem, which uses the same interval convention.

\subsection{Energy Consumption Estimate} \label{subsec:agent-energy}

As described in \Cref{eq:time-grid-carbon}, energy consumption and real-time carbon intensity are the two primary factors in carbon emission calculation. 
Carbon intensity has been widely studied, and databases are available for querying real-time carbon intensity across grids~\cite{watttime,eiaGridMonitor,entsoeTransparency,ele_maps}.
Therefore, the main challenge is to estimate the energy consumption of AI agents.

We begin with public data from prior studies and blog posts that have estimated the carbon emissions of serving AI models \cite{elsworth2025environmentalimpact,altman2025gentlesingularity,jegham2025howhungryai}. 
These studies span both lightweight chatbot applications and complex tasks such as reasoning.
For example, Sam Altman's blog mentions that a ChatGPT query consumes \Wh{0.34} of energy on average \cite{altman2025gentlesingularity} and Google's study shows that the median energy consumption of a Gemini application is \Wh{0.24} \cite{elsworth2025environmentalimpact}. 
As AI agent tasks typically use high-capability models and process long contexts, we focus on estimating energy consumption in such scenarios, based on a recent study of the energy consumption of the OpenAI API~\cite{jegham2025howhungryai}.
Specifically, their long, high-reasoning GPT-5 configuration assumes \SIx{10000} input tokens and \SIx{1500} output tokens, with a reported energy consumption of \Wh{33.8} per prompt. 
Since this study reports energy only at the prompt level, we use a simple effective per-token model:
\begin{equation}
    E_{\mathrm{prompt}} = e_{\mathrm{in}}N_{\mathrm{in}} + e_{\mathrm{out}}N_{\mathrm{out}},
\end{equation}
where $e_{\mathrm{in}}$ and $e_{\mathrm{out}}$ are the per-token input and output energy, and $N_{\mathrm{in}}$ and $N_{\mathrm{out}}$ are the corresponding numbers of input and output tokens. 

Based on typical metadata for token usage reported by AI service APIs, such as OpenAI's response usage fields for input, output, and cached tokens~\cite{openaiResponsesUsage,openaiPromptCaching}, we divide tokens into three categories: \textit{output tokens}, \textit{input tokens}, and \textit{cached input tokens}, 
where cached input tokens benefit from context caching as introduced in \Cref{subsec:llm}. 

Since the prior study focuses on individual prompts without caching, we first split the reported energy between input and output tokens. 
API pricing typically reflects the computational cost of model serving. Therefore, motivated by OpenAI's API pricing, where output tokens are priced at 6~times the input token price for GPT-5-class models~\cite{openaiPricing}, we assume that output tokens consume 6~times the energy of input tokens, \ie{} $e_{\mathrm{out}} = 6e_{\mathrm{in}}$.
This yields 
\begin{equation}
\begin{aligned}
e_{\mathrm{in}}  = 0.00178\ \mathrm{Wh/token}, \quad
e_{\mathrm{out}} = 0.01067\ \mathrm{Wh/token}.
\end{aligned}
\end{equation}
For cached input tokens, we again use their price relative to regular input tokens, \ie{} 10\,\% of the regular input token price~\cite{openaiPricing}, leading to
\begin{equation}
    e_{\mathrm{in,cache}}=0.000178\ \mathrm{Wh/token}.
\end{equation}

\begin{table}[t]
    \centering
    \caption{WildClawBench tasks and categories \cite{ding2026wildclawbench}. }
    \label{tab:wildclawbench-openclaw-categories}
    \small
    \begin{tabular}{lcl}
        \toprule
        Category & \# Tasks & Example tasks \\
        \midrule
        Productivity Flow  & 10 & Paper digest, PDF analysis \\
        Code Intelligence  & 12 & Website development, code debugging \\
        Social Interaction & 6 & Chat extraction, calendar scheduling \\
        Search \& Retrieval & 11 & Data collection, fuzzy search \\
        Creative Synthesis & 11 & Slides generation, posters generation \\
        Safety Alignment & 10 & Skill injection filtering, password leak detection \\
        \bottomrule
    \end{tabular}
\end{table}

\section{Characterization of AI Agent Workloads } \label{sec:characterization}

Although prior studies have characterized various AI agent workloads \cite{agentbench,webarena,swebench,zhao2026computer,yuan2026agenticaiworkloadcharacteristics,chung2026agentx,bai2026agentsmoney}, their carbon emissions have not been well studied. 
Therefore, we characterize OpenClaw workloads from a carbon emission perspective and summarize the key workload properties that motivate carbon-aware optimizations as takeaways.



\subsection{Characterization Methodology}
\label{sec:characterization-methodology}

We use OpenClaw with GPT-5.4 model API to run representative AI agent tasks and collect their execution statistics for characterization.
The workloads are based on WildClawBench \cite{ding2026wildclawbench}, a benchmark suite that contains a total of 60 AI agent tasks spanning 6 categories: productivity flow, code intelligence, social interaction, search and retrieval, creative synthesis, and safety alignment.
\Cref{tab:wildclawbench-openclaw-categories} lists the number of tasks in each category together with a few representative tasks. 
The carbon intensity data covers the four grids in \Cref{tab:grid_ci} with a granularity of \SI{15}{\minute}.
Note that the characterization focuses on analyzing the agent workload behaviors; detailed evaluation methodologies are discussed in \Cref{subsec:methodology}.

\begin{figure}
\centering
\begin{minipage}[t]{0.49\linewidth}
    \centering
    \includegraphics[width=1\linewidth]{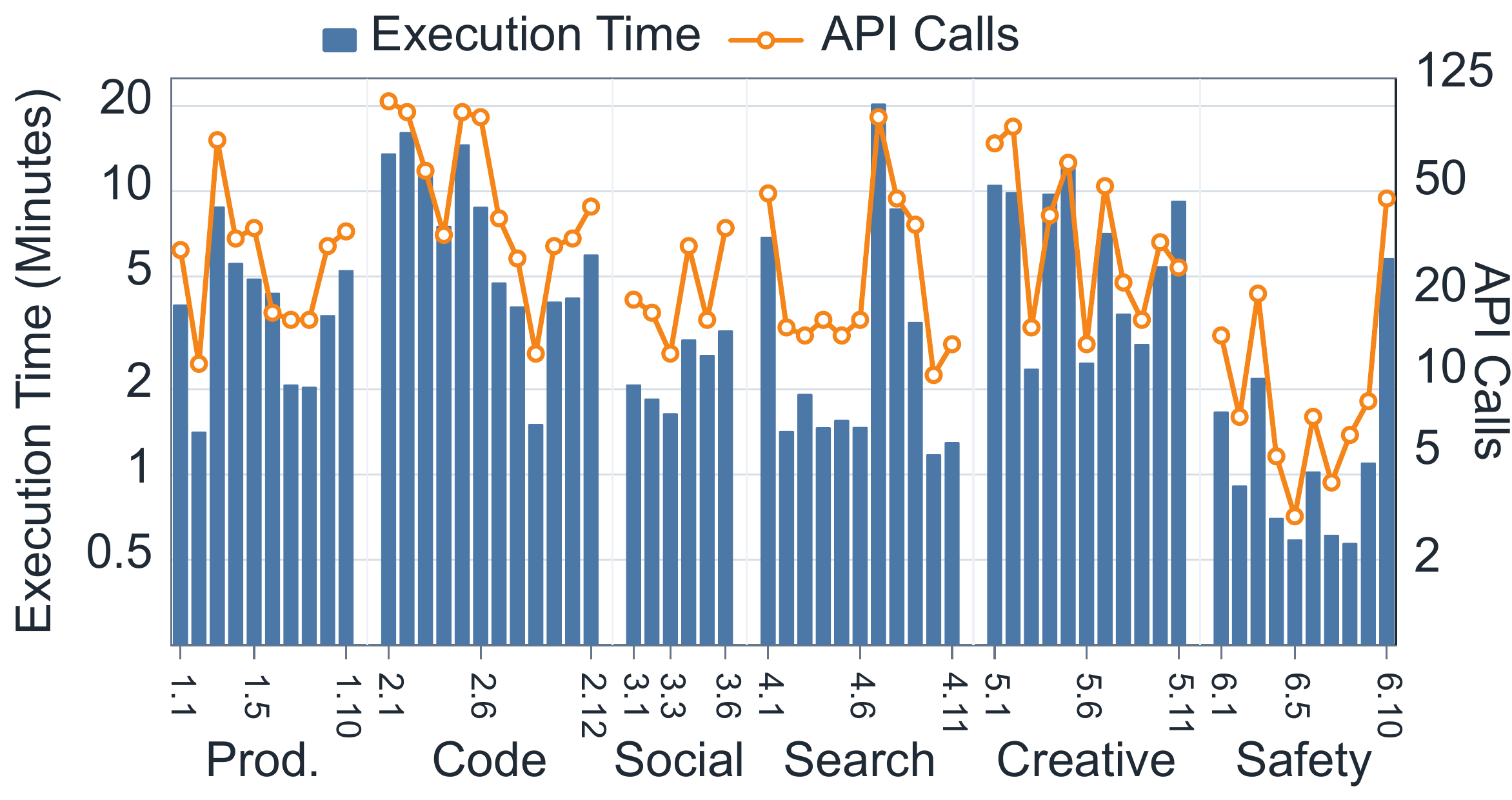}
    \caption{Execution time and number of API calls (log scale) for WildClawBench tasks.}
    \Description{}
    \label{fig:task_time}
\end{minipage}
\hfill
\begin{minipage}[t]{0.49\linewidth}
    \centering
    \includegraphics[width=1\textwidth]{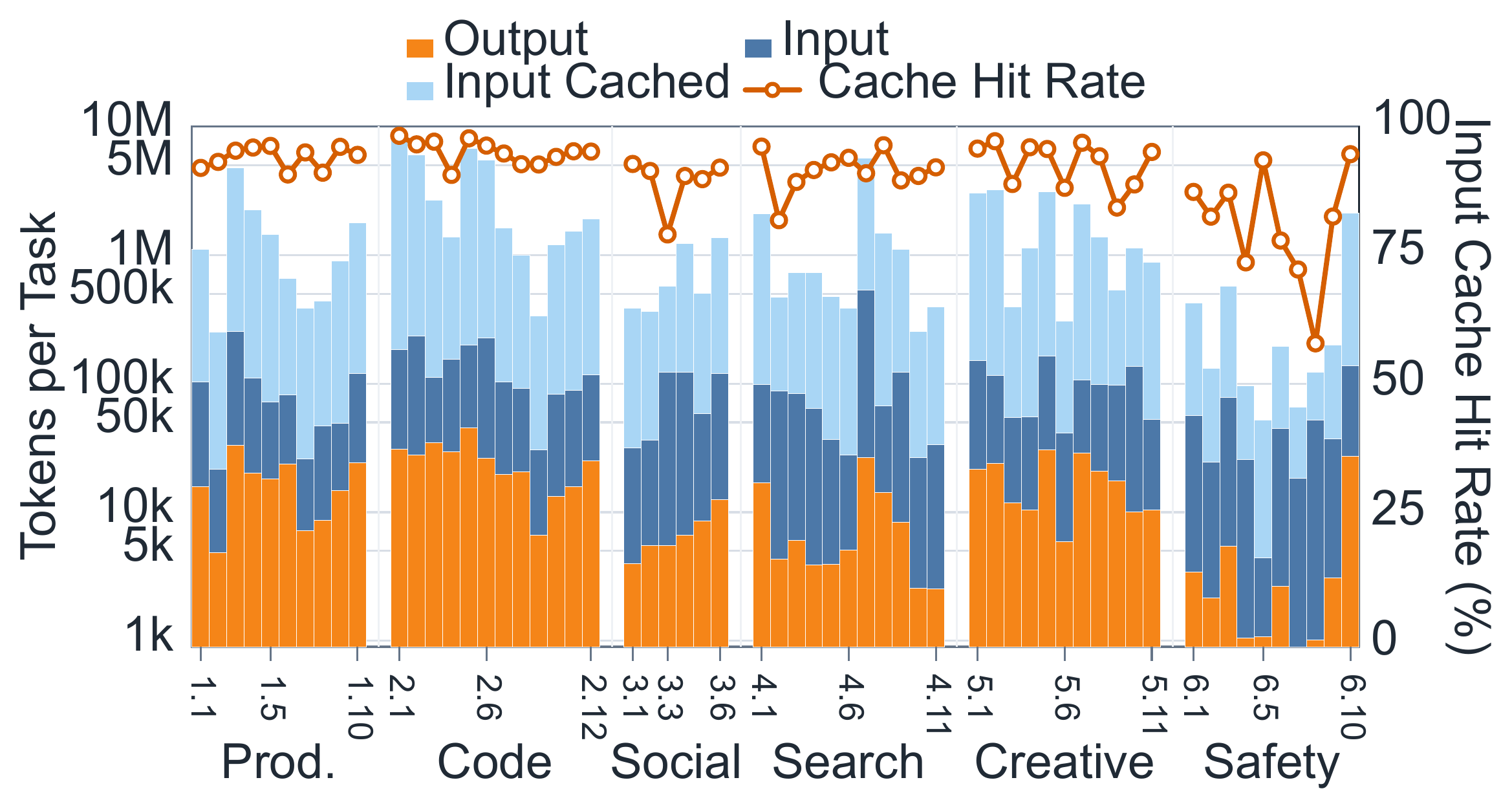}
    \caption{Token consumption (log scale) and cache hit rates of WildClawBench tasks.}
    \label{fig:task_token}
    \Description{}
\end{minipage}
\end{figure}

\subsection{Performance Characterization}
\label{subsec:performance_characterization}

We first evaluate the execution time of agent tasks.
\Cref{fig:task_time} plots the execution time of all 60 tasks across six categories on the left y-axis using a log scale.
Each task is labeled with a unique task ID on the x-axis.
We observe that agent tasks differ substantially in execution time.
Unlike a single user prompt, an agent task may expand into a variable sequence of model API calls, tool invocations, observations, retries, and validation steps.
As a result, more complex and challenging tasks generally take longer to complete.
For example, tasks in the code intelligence category are often more complex because they involve repeated testing and validation.

To better understand how OpenClaw interacts with the model API, we also plot the number of API calls on the right y-axis using a log scale.
Different tasks require substantially different numbers of API calls.
For example, some simple tasks in safety alignment, such as tasks 6.4--6.9, require no more than 8 API calls, while others require up to 103 calls (task 2.1 in code intelligence).
Despite this variation, execution time is largely consistent with the number of API calls.

We next analyze the number of tokens consumed by the agents.
\Cref{fig:task_token} presents the input, output, and cached input tokens for all tasks (left y-axis on a log scale) and cache hit rates (right y-axis on a linear scale).
We observe that agent tasks are generally input-heavy (\ie{} prefill-heavy): input tokens dominate the total token count across tasks (99.1\,\%), while output tokens account for a much smaller fraction (0.9\,\%).
A large portion of these input tokens (94.6\,\%) are cached input tokens, indicating frequent context reuse across repeated model API calls in agent workflows.
This observation is also consistent with recent studies on AI agents~\cite{zhao2026computer,yuan2026agenticaiworkloadcharacteristics,chung2026agentx}.
Overall, token consumption is largely consistent with execution time and the number of API calls.
However, it still varies substantially across tasks, ranging from 52k (task 6.5 in safety alignment) to 8.4M (task 2.1 in code intelligence) tokens, indicating that agent tasks impose highly diverse demands on model API services.

\begin{takeawaybox}{takeaway:agent-resource}{AI agent resource consumption varies}
Agent tasks vary widely in execution time, model API calls, and token consumption.
Because these tasks are usually input-heavy, most input tokens hit the context cache.
Carbon-aware optimizations need to account for execution time in order to meet user-specified deadlines.
\end{takeawaybox}

\subsection{Carbon Emission Characterization}
\label{subsec:emission-characterization}

\begin{figure}
\begin{minipage}[t]{0.49\linewidth}
    \centering
    \includegraphics[width=\linewidth]{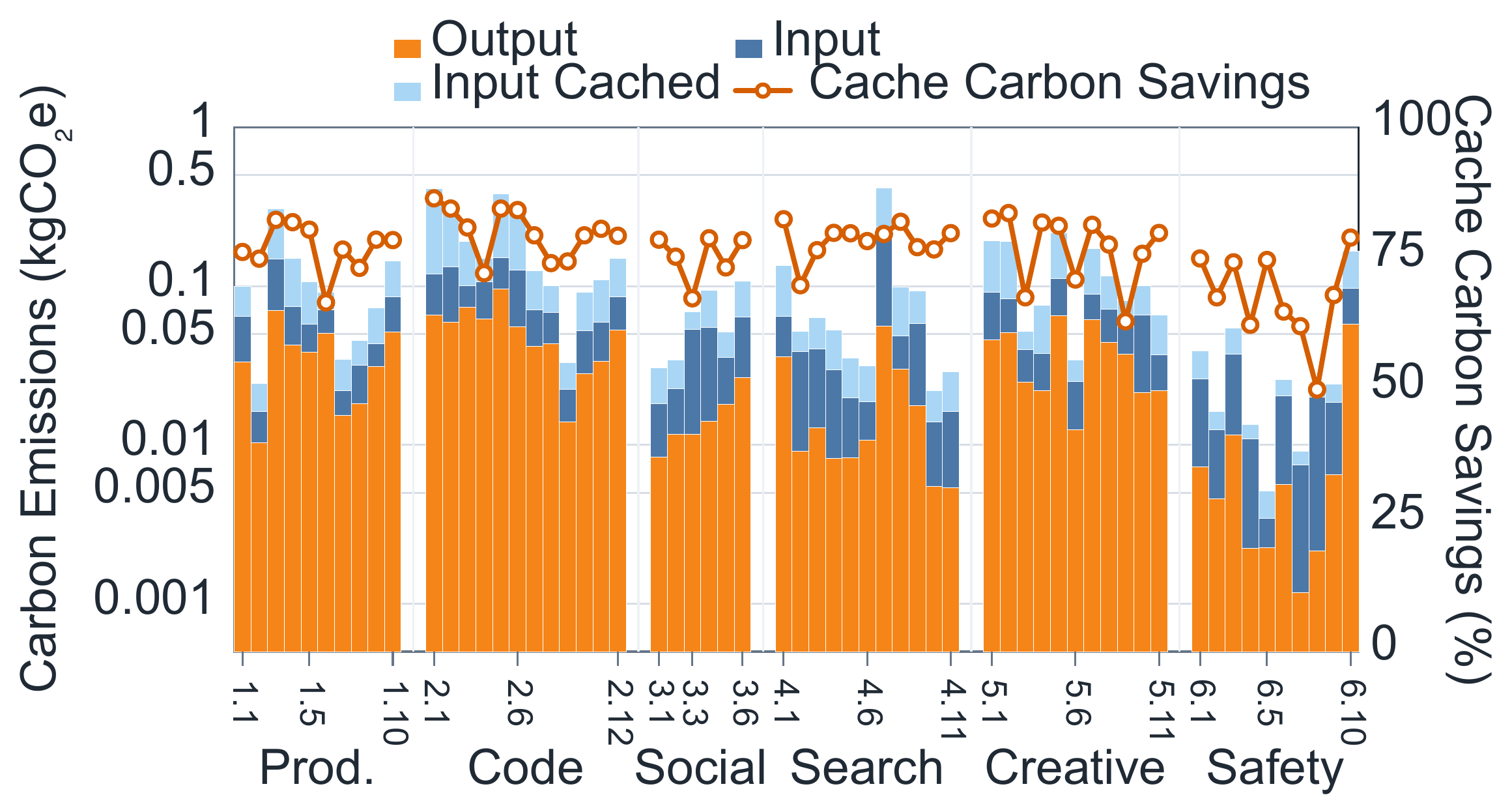}
    \caption{Carbon emissions (log scale) and savings from cache (linear scale) of WildClawBench tasks. Tasks start in CISO at 00:00 UTC on Feb 24, 2026.}
    \label{fig:task_carbon}
    \Description{}
\end{minipage}
\hfill
\begin{minipage}[t]{0.49\linewidth}
    \centering
    \includegraphics[width=\linewidth]{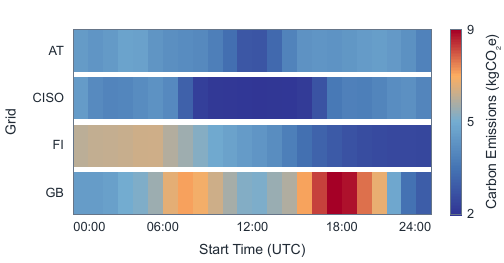}
    \caption{Carbon emissions from executing all WildClawBench tasks across regions and start times of Feb 24, 2026.}
    \label{fig:ci_emission_heatmap}
\end{minipage}
\end{figure}

We further study the carbon emissions of AI agent workloads using the estimated per-token energy consumption described in \Cref{subsec:agent-energy}. 
\Cref{fig:task_carbon} shows the corresponding carbon emissions from input, output, and cached input tokens for all tasks, using the left y-axis on a log scale. 
Because carbon intensity is grid- and time-dependent, we assume that the model is served in the CISO grid on Feb 24, 2026, and that each task starts at 00:00 UTC. 
Although output tokens have higher per-token carbon emissions, input tokens dominate the overall emissions because these agent tasks are input-heavy. 
On average, input tokens, including cached input tokens, contribute 72.8\,\% of carbon emissions, while output tokens contribute 27.2\,\%.
The total carbon emissions from executing the full WildClawBench are 6.6~\carbonKG{}.
At scale, these emissions are substantial.

\begin{takeawaybox}{takeaway:agent-carbon}{AI agents lead to high carbon emissions}
Agent tasks lead to significant carbon emissions. 
As agent tasks are mainly input-heavy, input tokens account for the majority of carbon emissions.
\end{takeawaybox}

\Cref{subsec:carbon-account-computing} shows that carbon emissions depend on both grid energy mix and execution time. 
Therefore, we study carbon emissions when serving AI models across different grids and at different times. 
\Cref{fig:ci_traces} shows 24-hour carbon intensity traces for four grids on Feb 24, 2026, whose curves interleave throughout the day. 
Each grid follows a different carbon intensity trend because of differences in its energy mix. 
Accordingly, \Cref{fig:ci_emission_heatmap} presents a heatmap of the corresponding carbon emissions from executing the full WildClawBench at different times of day across the four grids.
Because the carbon intensity curves interleave, the carbon-optimal grid for executing the AI agent workload changes over time.


\begin{takeawaybox}{takeaway:agent-carbon-time-grid}{Location and time matter}
Despite the high carbon emissions, the emission numbers are highly dependent on the grid where the AI model is deployed and the time when the task starts.
This implies potential carbon mitigation opportunities through geographic and temporal shifting of workloads, subject to task deadlines.
\end{takeawaybox}

Performance characterization in \Cref{subsec:performance_characterization} shows that a large fraction of input tokens benefit from caching.
\Cref{fig:task_carbon} shows the fraction of carbon emissions saved by caching on the right y-axis on a linear scale.
Because the carbon-savings percentage is normalized, it does not depend on the grid or execution time.
On average, context caching saves 80.7\,\% of the total carbon emissions, which is consistent with the high input token cache hit rate of 94.6\,\%. 

\begin{takeawaybox}{takeaway:cache}{Context caching is critical}
Context caching is critical to reducing carbon emissions, especially for input-heavy agent tasks.
Carbon emission optimizations should therefore consider the impact of caching.
\end{takeawaybox}

\section{High-Level Ideas}
\label{sec:high-level}

In this work, we introduce \name{}, a carbon optimizer for AI agent tasks on OpenClaw.
Prior work on carbon-aware optimizations commonly reduces emissions by delaying workloads to lower-carbon-intensity periods and/or by shifting them across locations~\cite{wiesner2021wait,souza2024casper,Caribou,CarbonExplorer,clover,gaia,li2025ecoserve}.
However, applying carbon-aware workload shifting to autonomous agents introduces two key challenges: (1)~agent workloads need to meet user-specified deadlines while deferring execution, and (2)~spatial shifting must balance cache overhead against the lower carbon intensity of other grids.


\subsection{Deadline and Progress Awareness}

\textbf{Challenge.}
Temporal shifting can reduce emissions by deferring agent workloads to low-carbon-intensity periods. However, the key challenge is determining whether deferred execution can still finish before the user-specified deadline.
Such estimates are difficult because agent execution is iterative and highly dynamic.
Each model API call or tool query depends on previous results and may change upon planning decisions, tool calls, external observations, and retries, causing substantial variation in execution time and token usage~\cite{yao2023react,mitchell2019realtimeplanning,xu-etal-2026-tps,bai2026agentsmoney}.


\noindent\textbf{Solution.}
Agent executions are typically divided into steps, where each step involves remote model API calls, tool usage, and other interactions with the environment~\cite{yao2023react,agentbench}.
Based on this property, \name{} first generates a conservative initial duration estimate before execution begins using the user's prompt and factors such as expected steps, step complexity, and model generation speed.
As execution proceeds, \name{} refines this estimate in flight by collecting progress observations after each step.
These observations reduce uncertainty in the remaining workload and allow \name{} to adapt future execution intervals based on the agent's progress and carbon intensity forecasts.
In this work, an execution interval is defined as a slice of the agent's computational workload, which we match with the carbon intensity granularity (\ie{} $\Delta t$ in \Cref{subsec:agent_carbon_accounting}) when selecting low-carbon execution times.

\begin{figure}
  \centering
  \includegraphics[width=\linewidth]{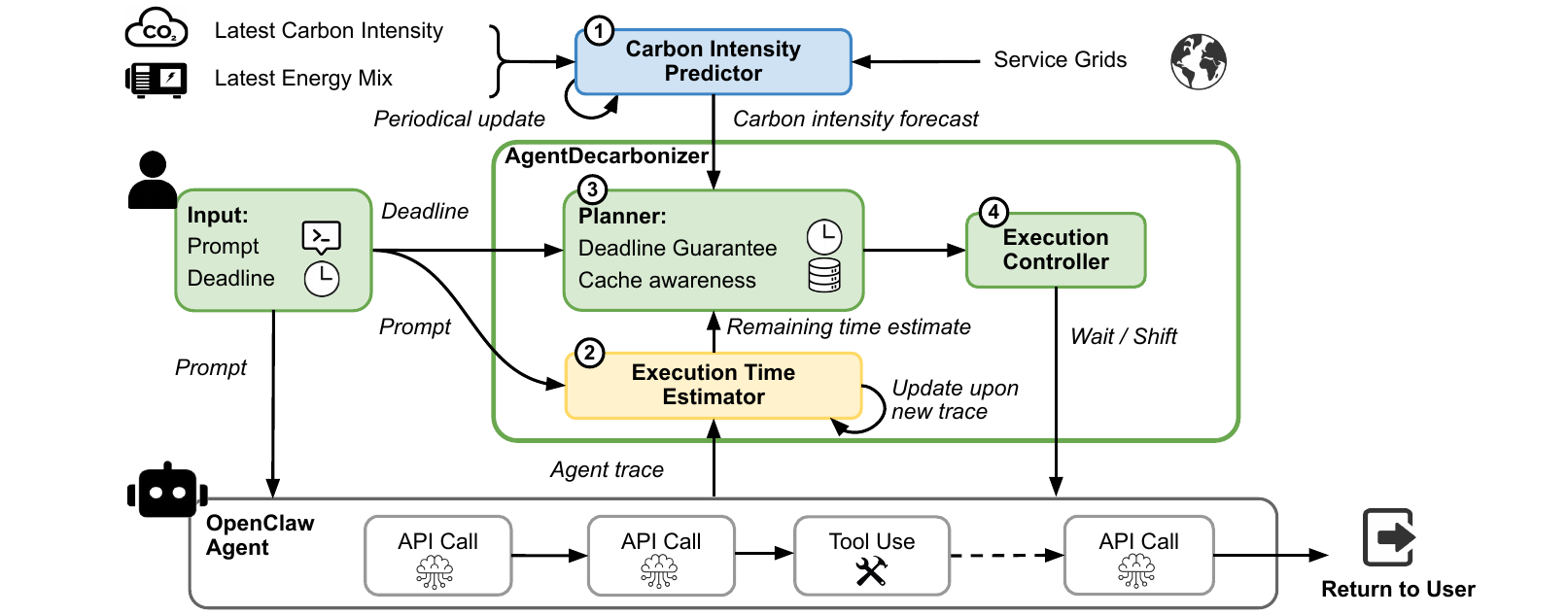}
  \caption{\name{} overview.}
  \Description{AgentDecarbonizer workflow from request submission to cache-aware shifting decisions and execution.}
  \label{fig:greenclaw-overview}
\end{figure}

\subsection{Context Cache Awareness}

\textbf{Challenge.}
Context caching reduces repeated computation by reusing previously processed prompt states, which is especially beneficial for long-context and reasoning-heavy requests.
However, these cached states can be large as they grow with model size and context length.
Therefore, serving systems typically keep caches local to a server or cluster rather than sharing them across regions~\cite{gim2024promptcache,cachedattention,hcache}; examples include Gemini context caching~\cite{googleGeminiContextCachingUse}, AWS Bedrock prompt caching~\cite{awsbedrockPromptCache}, and MoonCake~\cite{moocacke2024arxiv}.
As a result, spatially shifting execution to another grid can introduce cache misses and extra computation, potentially increasing carbon emissions even when the destination grid has lower carbon intensity.

\noindent\textbf{Solution.}
\name{} makes cache effects explicit in the planning process rather than assuming that every lower-carbon-intensity grid is beneficial.
The planner evaluates each candidate interval with a cache-aware carbon model: continuing in the same grid preserves reusable input context, while switching to another grid adds a cache miss penalty for recomputing previously processed context.
\name{} estimates this cache miss penalty using runtime statistics collected from past model API calls, such as the token consumption and cached input token metadata reported by the model API.
For each location shift decision, the planner compares the estimated carbon emissions of continuing in the current grid with the overhead of moving to a different grid due to cache misses.
A spatial workload shift is selected only when the estimated carbon reduction exceeds the recomputation overhead; otherwise, the planner continues execution or waits in the current grid.



\section{\name{}}
\label{sec:design}

We first present an overview of \name{} and then describe its components in detail.

\subsection{System Overview}
\label{sec:design-overview}

\name{} is a carbon optimizer that works on top of OpenClaw.
It jointly considers temporal shifting, deciding which intervals to use for execution or waiting, and spatial shifting, deciding which service grid should handle each selected execution interval, while preserving the original OpenClaw agent execution.
Figure~\ref{fig:greenclaw-overview} provides an overview of \name{}, which runs alongside the OpenClaw agent and periodically updates the execution plan.
The \textit{Carbon Intensity Predictor} forecasts future carbon intensity for each grid using recent carbon intensity and energy mix measurements (step~1).
The \textit{Execution Time Estimator} produces an initial, conservative execution time estimate from the user prompt and refines the remaining time estimate from OpenClaw traces as API calls and tool invocations complete (step~2).
The \textit{Planner} combines the deadline, carbon intensity forecasts, remaining time estimate, and cache miss cost model to choose low-carbon-intensity candidate intervals and service grids, including cache-aware shifting decisions (step~3).
The \textit{Execution Controller} applies the resulting plan by waiting for selected intervals or shifting execution to the grid specified by the planner, while collecting progress and cache statistics that trigger later replanning (step~4).

\subsection{Execution Time Estimator}
\label{sec:design-duration-estimator}

The estimator is built around a lightweight local LLM, a 4B Gemma~4 model in our implementation (see \Cref{subsec:methodology}), that analyzes user requests, extracts task characteristics, and derives time estimates, without actually executing the task.
Figure~\ref{fig:task-duration-estimator} presents the key steps.
As introduced in \Cref{sec:design-overview}, the estimator first generates an initial estimate using the input prompt, and then refines the estimate during execution by collecting the latest execution statistics from OpenClaw.
We next describe the initial estimation pass and then highlight the in-flight estimate update process.

\begin{figure}
  \centering
  \includegraphics[width=\linewidth]{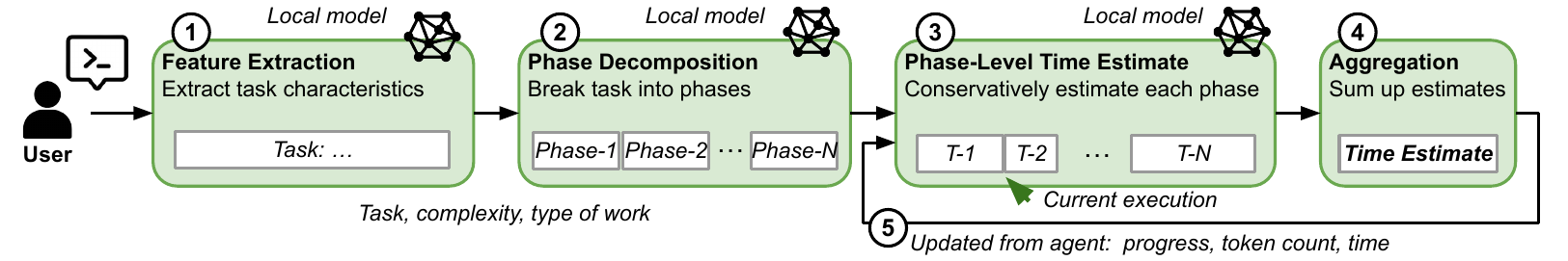}
  \caption{Execution time estimator workflow.}
  \Description{Estimator workflow from input preparation to feature extraction, execution plan decomposition, step-level duration estimation, aggregation, and runtime updates.}
  \label{fig:task-duration-estimator}
\end{figure}

\textbf{Estimator Workflow.}
\textbf{Step~1: Feature extraction.}
Motivated by prior characterizations of LLM and AI agent workloads, which identify factors such as context length, token generation, execution steps, and tool interactions as key runtime contributors~\cite{zhu2026tracelab,sweagent,openSWETraces}, the estimator extracts runtime-relevant task features from the original task input, metadata, and local signals. The exact features and prompt template are provided in Appendix~\ref{app:duration-estimator-step1-prompt}.
\textbf{Step~2: Phase decomposition.}
The local LLM uses the extracted task features to split the task into execution phases. This decomposition exposes the primary runtime contributors and uncertainty sources, enabling phase-level duration estimation.
Appendix~\ref{app:duration-estimator-step2-prompt} provides the corresponding prompt template for this phase decomposition step.
\textbf{Step~3: Phase-level time estimation.}
For each phase, the local LLM takes the phase description and runtime-relevant task features as input and predicts a conservative duration estimate $\widetilde{t}_i$. 
We aim to provide a generic method rather than one tailored to the WildClawBench workloads. Therefore, the prediction is guided by heuristic anchors from publicly available studies, including the expected step count ranges summarized in Appendix~\ref{app:duration-estimator-calibration}, for model calls, tool use, generation, validation, latency, and retries.
This is motivated by prior AI agent work, where tasks unfold as multi-step reasoning and tool-use trajectories~\cite{yao2023react,zhu2026tracelab}.
\textbf{Step~4: Aggregation.}
Finally, the estimator sums the conservative phase estimates to obtain the task-level duration $\widetilde{T}=\sum_i \widetilde{t}_i$, which is the execution time reserved before the deadline.

\textbf{Runtime progress update.}
Because agent execution is multi-step and highly dynamic, the initial estimate can become inaccurate as execution progresses.
Therefore, the estimator periodically updates the remaining duration based on observed execution progress.
It compares the original execution plan with runtime signals, including completed phases, executed tool calls, generated tokens, and recent API and message history, to estimate the fraction of planned work completed.
The remaining duration is then adjusted based on the observed execution speed.

\subsection{Cache-Aware Carbon Model}
\label{sec:design-cache-model}
During planning, \name{} models context cache as a location-dependent shifting penalty.
If execution stays in the same grid, the plan assumes reusable cached input context remains available; if execution shifts to another grid, \name{} uses token consumption and cached-token metadata to estimate the extra computation needed to rebuild the context.
This model reflects the LLM inference workflow: during prefill, an LLM processes input tokens and constructs reusable KV-cache state~\cite{pagedattention,agrawal2024sarathi}.
When shifting to a different grid makes the state unavailable, \name{} approximates the cache miss overhead as the prefill work needed to reprocess the previously cached input tokens.
Note that \name{} is a flexible framework which allows for other cache carbon models (discussed in \Cref{sec:related}).


\subsection{Carbon-Aware Planner}
\label{sec:design-shifting-planner}
The planner converts duration estimates, carbon intensity forecasts, and cache-aware carbon emissions into a placement plan.
By default, \name{} uses EnsembleCI~\cite{yan2025ensembleci} as the local carbon intensity predictor to forecast at a granularity of $\Delta t$, which is set to \SI{15}{\minute} in our implementation, and refreshes the forecasts once per hour during execution.
Like the \textit{Execution Time Estimator} in \Cref{sec:design-duration-estimator}, the planner also operates in two stages: \textit{initial planning} before execution and \textit{in-flight updates} while the agent is running.

\textbf{Initial planning.}
Upon receiving the task, \name{} obtains carbon intensity forecasts for all candidate grids over the user-specified deadline window.
Given the initial duration estimate $\widetilde{T}$ from the \textit{Execution Time Estimator}, the planner converts the required execution time into $K=\lceil\widetilde{T}/\Delta t\rceil$ execution intervals of length $\Delta t$ and selects the $K$ lowest-carbon, non-overlapping candidate intervals before the deadline.
Because no execution history exists yet, the initial plan does not include spatial shifting penalties.

\textbf{In-flight update.}
During execution, each planning refresh updates the remaining placement plan. 
The planner takes updated carbon intensity forecasts, triggers the \textit{Execution Time Estimator} to update the remaining duration from observed progress, and consumes runtime cache signals collected by the controller, such as input and cached token statistics.
It then reoptimizes the placement of the unfinished execution over the remaining deadline horizon.
In this optimization, the planner uses the cache-aware carbon model in Section~\ref{sec:design-cache-model} to balance the lower execution emissions available in other grids against the cache-recomputation carbon emissions caused by shifting away from the current grid.
The updated plan records the resulting temporal and spatial workload shifting decisions.

At either stage, if the time remaining before the deadline is shorter than the estimated execution time, the planner falls back to the default OpenClaw execution path.
Otherwise, it sends the resulting placement plan to the \textit{Execution Controller}.

\subsection{Execution Controller}
\label{sec:design-execution-controller}

The execution controller manages \name{}'s runtime flow according to the execution plan from the \textit{Carbon-Aware Planner}.
It starts execution in the selected candidate interval and grid, pauses between execution intervals when the plan calls for waiting, resumes from the saved execution state, and performs location shifting when the planner selects a different region.
Beyond controlling the execution, the controller leverages OpenClaw's default logging to collect runtime records, including agent messages, detailed token consumption (input, output, and cached input), and tool usage.
When an execution interval ends, the controller extracts these records, sends them to the \textit{Planner}, and waits for the updated plan before starting or resuming the next interval.

Currently, we implement \name{}'s spatial-shifting mechanism in a carbon-emissions simulator that assumes remote model requests can be routed to multiple service regions.
The main model we evaluate, GPT-5.4, does not provide a user-accessible API to choose among service locations; this feature is only available as a data residency option \cite{openaiAPIDataResidency}.
However, location shifting is a realistic capability because some cloud providers already expose location controls to users.
For example, Google Cloud supports location-specific model serving~\cite{googleVertexAILocations}.

\section{Theoretical Analysis}
\label{sec:theory}

This section formalizes the mechanisms described in \Cref{sec:design-shifting-planner} and \Cref{sec:design-execution-controller}.
We first formalize \name{}'s placement objective.
We then show that when cache recomputation depends only on whether the service region changes between consecutive execution intervals, this objective admits an exact dynamic-programming algorithm.
Finally, we prove that the resulting schedule meets the deadline whenever the conservative duration estimate covers the task's actual accumulated runtime.

\subsection{Carbon-Aware Execution Placement}
\label{sec:theory-problem}

A request arrives at time $a$ with deadline $d$ and candidate execution regions $\mathcal{R}$. As described in \Cref{sec:design-shifting-planner}, the carbon-intensity predictor produces forecasts at granularity $\Delta t$, which is set to \SI{15}{\minute} in our implementation. The planner uses the same granularity and treats each forecast time slot as a candidate execution interval. Specifically, it forms candidate intervals $I_t=[t,t+\Delta t)$ whose starting timestamps are $t\in\mathcal{H}=\{a+j\Delta t\mid j=0,1,\ldots,H-1\}$, where $H=\left\lfloor(d-a)/\Delta t\right\rfloor$. During each interval, \name{} may execute part of the agent task in one candidate region, or leave the interval idle and resume execution later. By construction, every candidate interval ends no later than the deadline.

Given the conservative duration estimate $\widetilde{T}$ from the \textit{Execution Time Estimator}, \name{} decomposes the agent task into $K=\lceil\widetilde{T}/\Delta t\rceil$ ordered execution intervals, each of length $\Delta t$, indexed by $k\in\{1,\ldots,K\}$.
These execution intervals are the consecutive chunks of work that must be allocated to candidate intervals.
\name{} preserves the original execution order of the task, but may defer execution mid-task and resume later for lower carbon emissions.

A placement plan is a sequence $\pi=\{(r_k,t_k)\}_{k=1}^{K}$, where $r_k\in\mathcal{R}$ denotes the region selected for execution interval $k$, and $t_k\in\mathcal{H}$ denotes the starting timestamp of the candidate interval in which it is placed.
Feasibility requires chronological order, $t_1<t_2<\cdots<t_K$.
If no feasible placement exists, \name{} falls back to the default OpenClaw execution.

Let $\widehat{E}^{(k)}_{r,t}$ be the predicted energy demand of running execution interval $k$ in region $r$ during interval $[t,t+\Delta t)$, and let $\widehat{CI}_{r,t}$ be the corresponding carbon intensity forecast.
The predicted carbon-emission term is $\widehat{E}^{(k)}_{r,t}\widehat{CI}_{r,t}$, matching the carbon accounting model in \Cref{eq:time-grid-carbon}.
Let $\pi_{1:k}=\{(r_i,t_i)\}_{i=1}^{k}$ denote the partial placement up to execution interval $k$.
For $k\ge2$, let $\Lambda^{(k)}(\pi_{1:k})\ge0$ denote the cache-recomputation overhead incurred when running execution interval $k$.
This term captures the extra carbon emissions from regenerating context or cache state after a cache miss; in the general formulation, it may depend on the previous placement history.

\name{} selects a deadline-feasible sequence of candidate intervals and the serving region for each selected interval.
The placement objective is to minimize predicted carbon emissions and cache-recomputation overhead, subject to deadline feasibility and chronological execution order:
\begin{tcolorbox}[colback=white,colframe=black!50,boxrule=0.5pt,left=0.6mm,right=0.6mm,top=0.5mm,bottom=0.5mm]
\begin{equation}
\begin{aligned}
    \min_{\pi}\quad
    J_{\mathrm{ord}}(\pi)
    &=\sum_{k=1}^{K}
      \widehat{E}^{(k)}_{r_k,t_k}\widehat{CI}_{r_k,t_k}
      +\sum_{k=2}^{K}\Lambda^{(k)}(\pi_{1:k}) \\
    \text{s.t.}\quad
    &r_k\in\mathcal{R},\quad t_k\in\mathcal{H}, && k=1,\ldots,K, \\
    &t_1<t_2<\cdots<t_K. &&
\end{aligned}
\label{eq:ordered-interval-objective}
\end{equation}
\end{tcolorbox}

\begin{table}[t]
	\centering\small
	\caption{Notation in the \name{} placement objective in \Cref{eq:ordered-interval-objective}.}
	\label{tab:interval-objective-notation}
	\begin{tabular}{@{}L{0.15\linewidth}L{0.8\linewidth}@{}}
		\toprule
		Symbol & Description \\
		\midrule
		$\mathcal{R}$ & Candidate execution regions. \\
		$\mathcal{H}$ & Deadline-feasible candidate intervals. \\
		$H$ & Number of deadline-feasible candidate intervals. \\
		$K$ & Number of ordered execution intervals induced by $\widetilde{T}$. \\
		$r_k$ & Region assigned to execution interval $k$. \\
		$t_k$ & Candidate interval assigned to execution interval $k$. \\
		$\widehat{E}^{(k)}_{r,t}$ & Predicted energy for execution interval $k$ in region $r$ during candidate interval $t$. \\
		$\widehat{CI}_{r,t}$ & Predicted carbon intensity for region $r$ during candidate interval $t$. \\
		$\Lambda^{(k)}(\pi_{1:k})$ & Cache-recomputation carbon emissions incurred during execution interval $k$. \\
		$\Gamma^{(k)}_{r,t}$ & Cache-recomputation penalty incurred when execution interval $k$ switches to grid $r$ during candidate interval $t$. \\
		\bottomrule
	\end{tabular}
\end{table}

\subsection{Dynamic Programming Under Cache Recomputation}
\label{sec:theory-dp}
In this section, we consider a local cache-recomputation structure and show that, under this structure, a dynamic-programming algorithm solves the placement objective in \Cref{eq:ordered-interval-objective} exactly.

\begin{assumption}[Local cache recomputation]
\label{assump:local-cache-cost}
	For every $k\ge 2$, the cache-recomputation overhead incurred in execution interval $k$ depends only on whether the service region changes between execution intervals $k-1$ and $k$. Intuitively, remaining in the same region preserves the context cache, whereas shifting to a different region causes a cache miss and requires the context to be regenerated. Therefore, the additional carbon overhead in execution interval $k$ depends only on the transition from interval $k-1$ to interval $k$, rather than on the full placement history. Specifically, there exists a nonnegative cache-recomputation carbon penalty $\Gamma^{(k)}_{r,t}\ge 0$ such that, for every partial placement $\pi_{1:k}$,
\begin{equation}
	\Lambda^{(k)}(\pi_{1:k}) =\mathbf{1}\{r_k\ne r_{k-1}\}\Gamma^{(k)}_{r_k,t_k}. \label{eq:local-cache-penalty} 
\end{equation}
\end{assumption}

Under \Cref{assump:local-cache-cost}, we formulate the following dynamic-programming algorithm.
Let $D^{(k)}_{r,t}$ be the minimum cost for the first $k$ execution intervals when execution interval $k$ runs in region $r$ during candidate interval $t$; infeasible states have value $+\infty$.
The base case is $D^{(1)}_{r,t}=\widehat{E}^{(1)}_{r,t}\widehat{CI}_{r,t}$.
For $k\ge2$, the recurrence is
\begin{equation}
    D^{(k)}_{r,t}
    =\widehat{E}^{(k)}_{r,t}\widehat{CI}_{r,t}
     +\min_{\substack{q\in\mathcal{R},\ s\in\mathcal{H}:\ s<t}}
        \left\{D^{(k-1)}_{q,s}
        +\mathbf{1}\{q\ne r\}\Gamma^{(k)}_{r,t}\right\}.
    \label{eq:interval-dp-recurrence}
\end{equation}
The recurrence searches over feasible predecessors: execution interval $k-1$ must finish in an earlier candidate interval $s<t$.
The transition cost is zero when the cache is preserved ($q=r$) and equals the cache-recomputation penalty when \name{} switches regions ($q\neq r$).
The optimal interval-level carbon estimate is $C^*=\min_{r\in\mathcal{R},\ t\in\mathcal{H}}D^{(K)}_{r,t}$.
When computing each minimum, the dynamic-programming algorithm stores a backpointer to the predecessor state $(q,s)$ that attained the minimum.
Tracing these backpointers from the optimal final state reconstructs the placement plan $\{(r_k,t_k)\}_{k=1}^{K}$.

\begin{theorem}[\name{} optimality and complexity]
	\label{thm:greenclaw-optimal}
	Under \Cref{assump:local-cache-cost}, the dynamic-programming algorithm returns a minimizer of \Cref{eq:ordered-interval-objective}. The recurrence can be evaluated in $O(K|\mathcal{R}|H)$ time and $O(|\mathcal{R}|H)$ working memory, plus backpointers for reconstructing the placement plan.
\end{theorem}

The dynamic-programming algorithm avoids repeatedly searching over all possible placements of the previous execution interval.
A naive evaluation checks every possible previous region and earlier candidate interval for every state, requiring $O(K|\mathcal{R}|^2H^2)$ time.
Under the local cache-recomputation structure, the previous-placement information can be summarized and reused.
Specifically, for each current state $(r,t)$ at execution interval $k$, the transition only needs two summaries over placements of execution interval $k-1$ in earlier candidate intervals $s<t$: the best placement that remains in the same region $r$, and the best placement that comes from any different region.
The first summary gives the cost of preserving the cache, while the second summary gives the best switching cost before adding the recomputation penalty $\Gamma^{(k)}_{r,t}$.
This allows each transition to be evaluated in constant time, reducing the overall complexity to $O(K|\mathcal{R}|H)$.
The details and proof of \Cref{thm:greenclaw-optimal} are provided in \Cref{app:greenclaw-optimal-proof}.

\subsection{Deadline Guarantee}
\label{sec:theory-deadline}

Let $T_{\mathrm{run}}$ denote the actual accumulated runtime of the task, excluding idle waiting between execution intervals.
Recall the conservative duration estimate $\widetilde{T}$.

\begin{theorem}[Deadline guarantee]
\label{thm:interval-deadline-guarantee}
Suppose the dynamic-programming algorithm returns a feasible \name{} schedule. If $T_{\mathrm{run}}\le\widetilde{T}$, then the schedule completes the task by the deadline $d$.
\end{theorem}

Recall that the \textit{Execution Time Estimator} in \Cref{sec:design-duration-estimator} decomposes the task into phases and assigns the $i$-th phase a conservative duration estimate $\widetilde{t}_i$, with $\widetilde{T}=\sum_i \widetilde{t}_i$. Let $t_i^{\mathrm{run}}$ denote the actual execution time spent on the corresponding phase. One sufficient condition for $T_{\mathrm{run}}\le \widetilde{T}$ is that $t_i^{\mathrm{run}}\le \widetilde{t}_i$ for every phase $i$. This motivates accounting for uncertainty during phase-level duration estimation and helps ensure that the reserved execution intervals provide enough time for the task to finish before the deadline.

\section{Evaluation} \label{sec:evaluation}

In this section, we first describe the evaluation methodology and then report the results.

\subsection{Methodology} \label{subsec:methodology}

\textbf{Local Platform.}
We use a Mac mini with an M4 processor and \SI{16}{\giga\byte} memory as the local platform to host OpenClaw and optimization mechanisms from \name{}.
We evaluate the overhead on the local platform in \Cref{subsec:local_overhead}.

\textbf{Models.}
We primarily evaluate GPT-5.4 as the backend model used by OpenClaw through OpenAI's API.
Additionally, we evaluate Gemini 3.1 Pro to compare \name{}'s carbon savings across models in \Cref{subsec:gemini_savings}.
On the local platform, we use a 4-bit-quantized Gemma 4 E4B model for execution-time estimation~\cite{googleGemma4ModelCard}; this lightweight model is suitable for mobile and edge devices. 
The local model is served using Ollama~\cite{ollamaDocs}, whose built-in prompt and context caching reduce inference time because the estimator prompts share substantial context for task descriptions. 
In our experiment, the estimator has no prior knowledge of the benchmark to ensure fair evaluation.

\textbf{Benchmark and Evaluation Setup.}
As in \Cref{sec:characterization}, we evaluate all 60 AI agent tasks in WildClawBench, which span six categories.
Because individual benchmark tasks have relatively short execution times, whereas real-world agent workflows are increasingly long-running, we combine shorter tasks within each category into a single workload.
This yields six workloads, one per category.
Within each task, the backend model maintains a context cache. Because the tasks within a workload are independent, however, the cache is not reused across tasks; we assume a cache lifetime of 24 hours.
For each workload, we record an OpenClaw execution trace.
Then, we develop a simulator that replays the traces to reflect different carbon intensity scenarios, and optimization decisions made by \name{} at runtime. 
We evaluate \name{} under four workload deadlines: 3, 6, 12, and 24 hours. Tighter deadlines reflect more on-demand execution, while longer deadlines reflect more flexible workload scheduling.

\textbf{Carbon-Intensity Traces.}
We use carbon-intensity traces for the four grids introduced in \Cref{tab:grid_ci}: CISO data from EIA~\cite{eiaGridMonitor}, AT and FI data from ENTSO-E~\cite{entsoeTransparency}, and GB data from NESO~\cite{neso}.
We use two years of traces for training (May 2023--May 2025) and one month of traces (February 2026) for evaluation.
We evaluate 24 candidate start times per day over this 1-month testing trace, with one start time per hour, and compute the emissions for each start time. We then average these emissions over the 1-month trace.

\textbf{Baselines.}
Using the carbon emissions calculated as described above, we compare \name{} against two baselines:
(1)~\textbf{Current-optimal:} The carbon emissions from the original OpenClaw execution when run on the grid with the lowest carbon intensity at the workload start time.
We use this baseline as the primary comparison in this work. 
(2)~\textbf{Average:} The average carbon emissions when running the agent workloads on each grid, which reflects a carbon-agnostic scenario.

\subsection{Overall Carbon Emission Savings} 
\label{subsec:overall_savings}

\textbf{Carbon Savings.}
\Cref{fig:gpt_savings} presents the carbon emission savings when using GPT-5.4 as the backend model under different workload deadlines.
We observe that as the deadline becomes more flexible, the carbon savings increase, as \name{} has more flexibility to defer and shift the workload to other grids. 
Compared with the current-optimal and average baselines, \name{} achieves average carbon savings of 3.3--37.5\,\% and 34.9--57.9\,\%, respectively.
\Cref{fig:grid_usage} summarizes the fraction of execution time assigned to each grid under different deadlines.  
The results are consistent with the carbon intensity distributions shown in \Cref{fig:ci-box}. 
Notably, CISO is used most often, and its usage increases as the deadline becomes longer. This is because CISO has a wide daily carbon intensity distribution, with lower carbon intensity values than the other grids. Thus, \name{} uses CISO more often to leverage its low-carbon-intensity periods.

\begin{figure}
    \centering
    \begin{subfigure}[t]{0.5\linewidth}
        \centering
        \includegraphics[width=\linewidth]{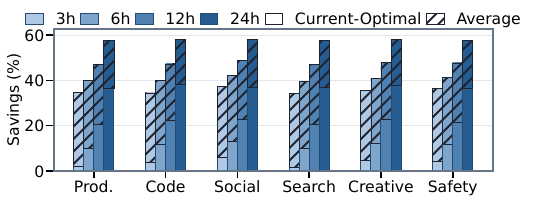}
        \caption{}
        \label{fig:gpt_savings}
    \end{subfigure}
    \hfill
    \begin{subfigure}[t]{0.24\linewidth}
        \centering
        \includegraphics[width=\linewidth]{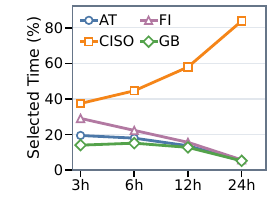}
        \caption{}
        \label{fig:grid_usage}
    \end{subfigure}
    \hfill
    \begin{subfigure}[t]{0.24\linewidth}
        \centering
        \includegraphics[width=\linewidth]{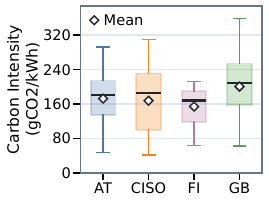}
        \caption{}
        \label{fig:ci-box}
    \end{subfigure}
    \caption{(a) Overall carbon savings, (b) grid usage, and (c) carbon intensity distribution. Model is GPT-5.4.}
    \Description{}
\end{figure}

\begin{figure}
    \centering
    \includegraphics[width=\linewidth]{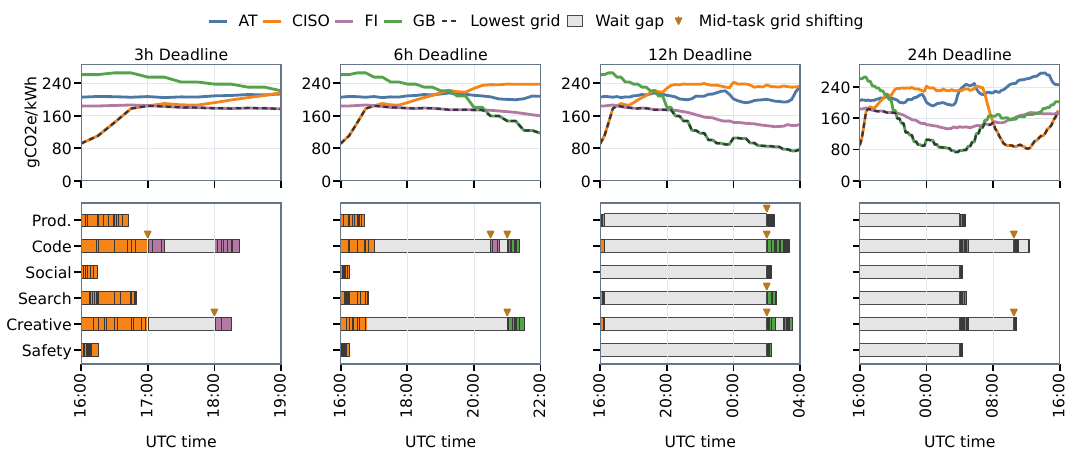}
    \caption{Example of agent execution optimized by \name{}. Start time Feb 4, 2026 at 16:00 UTC. }
    \label{fig:timeline_example}
    \Description{}
\end{figure}

\textbf{Timeline Examples.}
\Cref{fig:timeline_example} shows examples of agent workload execution on Feb 4, 2026, starting at 16:00 UTC, under four different deadlines.
We first observe that, when the deadline is tight, \name{} avoids introducing long waiting periods to meet the deadline. As the deadline becomes more flexible, \name{} spends more time waiting for lower-carbon periods.
Second, we find that mid-task migration occurs when there is a major change in carbon intensity, as indicated by the triangle marker.
For example, in the 3-hour deadline case, the workload is shifted from CISO to AT at 17:00 UTC due to the predicted increase in CISO's carbon intensity. 
Finally, we observe that cross-grid workload shifting is more frequent for long-running tasks, \eg{} Code and Creative.
Short-running workloads can often wait for a low-carbon period and then complete on a single grid, avoiding cache recomputation.
By contrast, long-running workloads may benefit from migrating before completion because more work remains after the shift.

\subsection{Comparison across Models} \label{subsec:gemini_savings}

\begin{figure}
    \centering
    \begin{subfigure}[t]{0.32\linewidth}
    \includegraphics[width=1\linewidth]{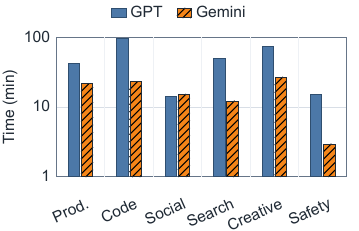}
    \caption{} 
    \label{fig:gemini_time}
    \end{subfigure}
    \hfill
    \begin{subfigure}[t]{0.32\linewidth}
    \includegraphics[width=1\linewidth]{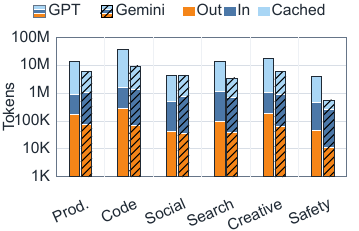}
    \caption{}
    \label{fig:gemini_token}
    \end{subfigure}
    \hfill
    \begin{subfigure}[t]{0.32\linewidth}
    \includegraphics[width=1\linewidth]{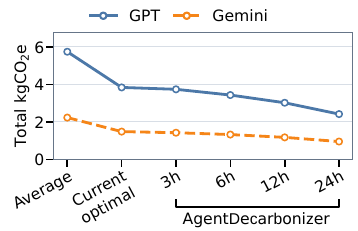}
    \caption{}
    \label{fig:gemini_carbon}
    \end{subfigure}
    \caption{(a) Execution time, (b) token usage, and (c) carbon emissions of Gemini 3.1 Pro vs GPT-5.4.}
    \label{fig:gemini_comparison}
    \Description{}
\end{figure}

We evaluate Gemini 3.1 Pro as a second model API to assess the effectiveness of \name{} across different models. 
\Cref{fig:gemini_time,fig:gemini_token} show the per-workload execution times and token usage, respectively, for both GPT-5.4 and Gemini 3.1 Pro.
Gemini 3.1 Pro has shorter execution times and lower token usage than GPT-5.4.

Because public data and studies on the energy consumption and carbon emissions of Gemini Pro series models are unavailable, we estimate the energy consumption of input, output, and cached input tokens using the same method as in \Cref{subsec:agent-energy}, based on the relative pricing differences between GPT-5.4 and Gemini 3.1 Pro~\cite{googleGeminiPricing,openaiPricing}, \ie{} Gemini 3.1 Pro costs \$2 per million input tokens compared with \$2.5 for GPT-5.4.
Using this estimate, \Cref{fig:gemini_carbon} compares the estimated total carbon emissions of the two models and shows that Gemini 3.1 Pro has lower estimated carbon emissions.
However, the carbon saving ratios remain similar: compared with the current-optimal and average baselines, \name{} reduces carbon emissions by 4--35.8\,\% and 35.2--57.3\,\%, respectively, under different deadlines. 
This experiment demonstrates that \name{} can reduce carbon emissions from AI-agent workloads across different backend models.


\begin{figure}
    \centering
    \includegraphics[width=\linewidth]{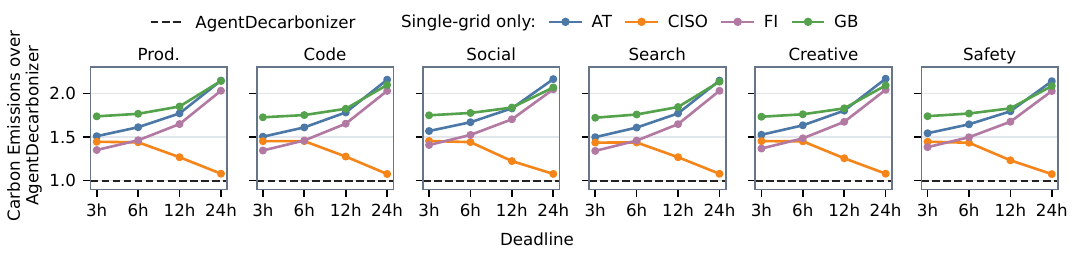}
    \caption{Emissions of single-grid only scenarios compared to \name{}. The model API is GPT-5.4. }
    \label{fig:single_grid}
    \Description{}
\end{figure}

\subsection{Ablation Studies} \label{subsec:ablation}

In this section, we conduct ablation studies to identify the sources of carbon savings.

\subsubsection{Impact of Multi-grid Migration}
We first study the impact of shifting across grids by comparing against a single-grid scenario that only waits for lower-carbon-intensity periods without shifting.
\Cref{fig:single_grid} presents the average carbon emissions of the single-grid scenario, normalized to \name{}, under different deadlines. 
We observe that shifting across grids in \name{} reduces carbon emissions by 34.1--45.7\,\% on average across workloads, depending on the deadline.
For most grids, longer deadlines allow \name{} to save more carbon.
The only exception is CISO, whose relative emissions compared with \name{} decrease as the deadline increases. 
As also shown in \Cref{fig:ci-box}, CISO has the highest carbon intensity variation, making it possible to mitigate a significant fraction of emissions by waiting alone. 
Overall, this experiment shows that shifting agent workloads across grids is a major source of carbon reduction.

\subsubsection{Impact of Carbon Intensity Predictor}

We then evaluate the impact of the carbon intensity predictor~\cite{yan2025ensembleci} used by \name{} by comparing it against an ideal scenario that uses ground-truth carbon intensities.
In addition, we compare \name{}'s hourly updated dynamic prediction with a static, one-time prediction.
\Cref{fig:ideal_carbon_predictor_savings} compares the carbon emission savings of these predictors over the default \textit{current-optimal} baseline. 
We observe that \name{} achieves carbon savings close to the ideal scenario, while the one-time prediction has lower carbon savings under longer deadlines because it fails to track carbon intensity changes over longer durations. 
This is because prediction errors grow with the horizon, reaching an average MAPE of 16.3\,\% at 24 hours, while next-hour predictions have a low MAPE of 3.3\,\%.
By updating predictions hourly, \name{}'s planner benefits from more accurate near-term carbon intensity estimates. 
This experiment demonstrates that, although carbon intensity prediction is imperfect, \name{} remains robust to prediction errors and substantially outperforms the one-time predictor under longer deadlines.

\subsubsection{Impact of the Execution-Time Estimator}

We finally evaluate the impact of the execution time estimator. 
We introduce an ideal scenario that uses the ground-truth execution time and a worst-case scenario that uses the execution time of the slowest task as a fixed estimate.
\Cref{fig:ideal_time_estimate} compares these scenarios by reporting their carbon emission savings over the default \textit{current-optimal} baseline. 
Overall, \name{} is close to the ideal scenario. 
We also compare the actual execution time with \name{}'s conservative estimate in \Cref{fig:actual_vs_estimate}. 
For all workloads, the estimated time is longer than the actual time, demonstrating that the estimate is conservative enough to ensure that workloads complete before their deadlines.
Compared to longer-running workloads, shorter workloads have larger slack between the estimated and actual execution times because the estimator uses a generic and conservative design.
In our evaluation, all workloads complete before the user-specified deadlines.

\begin{figure}[t]
    \centering
    \begin{minipage}[t]{0.48\linewidth}
        \centering
        \includegraphics[width=\linewidth]{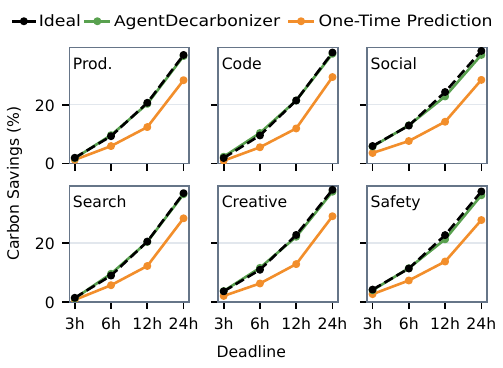}
        \caption{Carbon savings compared to ideal and one-time carbon intensity predictors. Model API: GPT-5.4.}
        \label{fig:ideal_carbon_predictor_savings}
        \Description{}
    \end{minipage}
    \hfill
    \begin{minipage}[t]{0.48\linewidth}
        \centering
        \includegraphics[width=\linewidth]{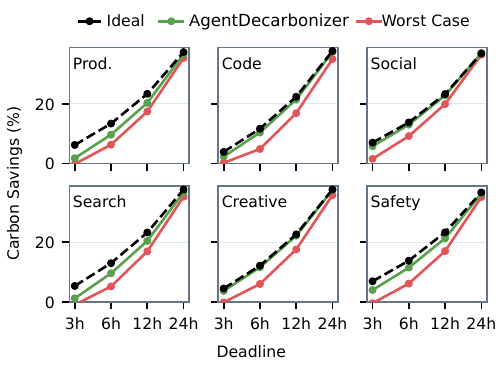}
        \caption{Carbon savings compared to ideal and worst-case execution time estimates. Model API: GPT-5.4.}
        \label{fig:ideal_time_estimate}
        \Description{}
    \end{minipage}
\end{figure}

\subsection{\name{} Overhead on Local Device} \label{subsec:local_overhead}

\begin{figure}
    \centering
    \begin{minipage}[t]{0.3\linewidth}
        \centering
        \includegraphics[width=\linewidth]{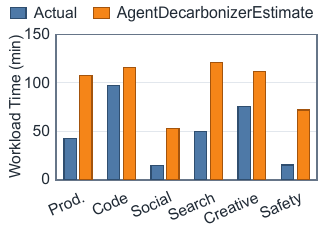}
        \caption{\name{}'s estimate vs. actual time. Model API: GPT-5.4.}
        \label{fig:actual_vs_estimate}
        \Description{}
    \end{minipage}
    \hfill
        \begin{minipage}[t]{0.45\linewidth}
            \centering
            \includegraphics[width=\linewidth]{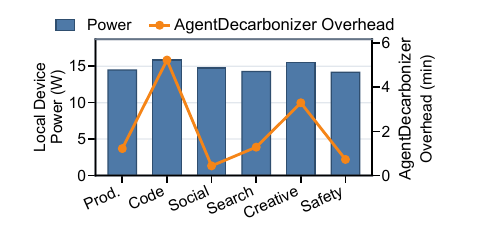}
            \caption{Power consumption and \name{} overhead on local device. }
            \label{fig:local_power}
        \end{minipage}
        \hfill
        \begin{minipage}[t]{0.2\linewidth}
            \centering
            \includegraphics[width=\linewidth]{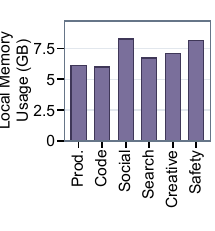}
            \caption{Memory usage on local device.}
            \label{fig:local_memory}
        \end{minipage}
\end{figure}

In this experiment, we evaluate the overhead of \name{} on the local Mac mini device.
We measure the power of the local Mac mini device using \texttt{powermetrics}~\cite{apple_powermetrics} every \SI{200}{\milli\second}.
\Cref{fig:local_power} presents \name{}'s average power consumption and total execution time. The local Gemma model used for execution-time estimation is the main source of overhead.
Overall, \name{}'s execution overhead is only \SIx{0.7}--\SI{5.2}{\minute} with an average power of \SI{14.9}{\watt}. 
This overhead is low compared to the long-running agent tasks. 
\Cref{fig:local_memory} further presents the peak memory usage. On average, the memory usage is only \SI{6.5}{\giga\byte}, mainly from the local Gemma model, which can fit into most PCs and laptops. 

\section{Related Work and Discussion} \label{sec:related}

In this section, we discuss related work. 

\textbf{Carbon-aware optimizations.}
Carbon-aware computing systems reduce operational emissions by exploiting variation in grid carbon intensity over time and across locations.
Temporal shifting delays flexible workloads until cleaner electricity is available, while spatial shifting moves work to lower-carbon regions~\cite{wiesner2021wait,souza2024casper,Caribou,CarbonExplorer,gaia,zhang2026carbonawarecompute}.
These systems establish the general opportunity for carbon-aware load shifting.
While \name{} builds on these directions, it addresses new challenges in reducing carbon emissions from AI-agent workloads, including meeting user deadlines during temporal shifting and accounting for context-cache recomputation overhead after spatial shifts.

\textbf{Carbon and Energy Analysis for AI Systems.}
Recent work has studied carbon mitigation, energy modeling, and serving optimization for AI and LLM systems, including model-, token-, and prompt-level energy estimates~\cite{li2025ecoserve,greenLLM,stojkovic2025dynamollm,faiz2024llmcarbon,samsi2023fromwordstowatt,zhang2026carbonawarecompute,nguyen2024towards}.
These approaches are complementary to \name{}.
Rather than changing the serving system, \name{} optimizes agent execution while preserving the original model API and agent behaviors.
In this work, we develop our energy estimate using a prior study \cite{jegham2025howhungryai} and vendor pricing. 
However, \name{} is not tied to a specific energy model; more precise and up-to-date estimates can replace the current values to improve carbon optimization.


\textbf{Context Caching.}
LLM serving systems increasingly rely on prompt and context reuse to accelerate long-context workloads~\cite{gim2024promptcache,cachedattention,yao2024cacheblend,yan2025contextcache,yu2025smartcache}.
In the current \name{} implementation, we assume that the context cache is local to a serving backend, \eg{} a server or cluster. This assumption is consistent with Gemini's context cache~\cite{googleGeminiContextCachingUse,googleGeminiContextCaching}, AWS Bedrock~\cite{awsbedrockPromptCache}, MoonCake~\cite{moocacke2024arxiv}, and prior studies~\cite{cachedattention,hcache,gim2024promptcache}. 
As a result, shifting agent execution to another grid can lose cache locality and require recomputation, with costs that grow with model size, context length, and generated reasoning history.
There have been recent proposals that stream context through the network~\cite{cachegen,lmcache}.
\name{} can support these techniques by using a lower shifting penalty in its cache carbon model when the serving system supports cross-region cache transfer.

\textbf{AI Agent Execution Time Prediction.}
Recent studies predict AI agent workloads, resource use, and token consumption to support budgeting, scheduling, and resource control for agent executions~\cite{bai2026agentsmoney,salim2026tokenomics,fan2025sweeffi,zheng2026agentcgroup}.
\name{} differs from general prediction work by using conservative estimates that are updated in flight to make deadline-aware load-shifting decisions during agent execution.
We also implement \name{}'s execution time estimator as a modular component, allowing future predictors to improve the planner and make better shifting decisions.
Besides using a more accurate estimator, the existing estimator's phase categories, step counts, per-step token estimates, and generation speed assumptions can also be customized for specific agent tasks, models, and deployment settings.  
In this work, we use values from publicly available datasets and research to keep the estimator generic.

\section{Conclusions}
\label{sec:conclusion}

This paper aims to mitigate the carbon emissions of autonomous AI agents. 
We first characterize OpenClaw workloads using WildClawBench to understand the carbon emissions of AI agents.
Based on these characteristics, we present \name{}, a carbon optimizer that shifts agent execution to lower-carbon time periods and grids while meeting user deadlines.
\name{} accounts for carbon savings from context caching and avoids spatial shifts when cache recomputation would outweigh the carbon benefit.
Our evaluation shows that \name{} reduces carbon emissions from agent workloads by up to 57.9\,\% while meeting user-specified deadlines.

\bibliographystyle{ACM-Reference-Format}
\bibliography{bib/ml,bib/carbon,bib/sys,bib/misc}

\appendix
\section{Appendix}
\label{app:theory-proofs}

\subsection{Execution-Time Estimator Calibration Categories}
\label{app:duration-estimator-calibration}

\Cref{tab:time-estimate-category} summarizes the heuristic calibration categories used by \name{}'s execution time estimator.
For each category, we use the typical step count or range reported in recent studies, as listed in the table.
The estimator first adjusts the number of steps relative to this baseline according to the predicted task complexity.
It then converts the resulting step count to an approximate wall-clock time using 267.6 tokens per step from Open-SWE-Traces~\cite{openSWETraces}.
To estimate API generation time, we use conservative rates derived from OpenRouter's model-speed statistics: 74 tokens/s for GPT-5.4 and 105 tokens/s for Gemini 3.1 Pro~\cite{openrouterGPT54,openrouterGemini31}.
Because OpenRouter reports real-time statistics that can vary over time, we use conservative rates rather than instantaneous measurements.


\begin{table}[h]
    \centering
    \footnotesize
    \setlength{\tabcolsep}{1.5pt}
    \caption{Phase categories and number of steps per phase.} \label{tab:time-estimate-category}
    \begin{tabular}{lll}
        \toprule
        Category & Typical steps/range & Reference \\
        \midrule
        Single-context text transformation & 1 & IFEval~\cite{zhou2023ifeval} \\
        Structured strict-format output & 4 & ToolLLM~\cite{toolllm} \\
        External retrieval/browser/API workflow & 7.3 & Mind2Web~\cite{mind2web} \\
        Multi-turn communication/coordination & 8.93--15.39 & MultiWOZ~\cite{multiwoz} \\
        Repository/code/debug workflow & 14.71 & SWE-agent~\cite{sweagent} \\
        Multimodal visual/computer-use workflow & 15 & OSWorld~\cite{osworld} \\
        Generation plus refinement workflow & 9 & Self-Refine~\cite{selfrefine} \\
        Policy/adversarial-instruction workflow & 6.3--9.3 & AgentDyn~\cite{agentdyn} \\
        Data analysis/computation workflow & 5.8--11.6 & DA-Code~\cite{huang2024dacode} \\
        \bottomrule
    \end{tabular}
\end{table}

Note that \Cref{tab:time-estimate-category} covers common categories that describe execution phases. 
We did not profile the benchmark tasks to fine-tune these numbers, in order to present a generic method. 
However, \name{}'s estimator can incorporate additional categories and customize the number of steps depending on the target tasks.
Likewise, the per-step token count and generation time can be customized based on the specific model used by the workload.

\subsection{Feature-Extraction Prompt}
\label{app:duration-estimator-step1-prompt}

\noindent System prompt for the local model to extract features from the task.\par
\begin{lstlisting}[style=promptblock]
You extract compact facts for execution time prediction.
Only agent_prompt_section contains the task input.
Model-trace features summarize the target model's trace shape, not the estimator LLM's trace.
Do not solve the task.
Return only valid minified JSON. No markdown. No comments.
Begin with { and ensure that all braces are closed.
Use integer seconds.
Keep arrays short: at most 6 items, except for the phases array.
Use plain labels, not long paragraphs.
You may use the 'thought' key to reason before other keys, but do not wrap the other schema keys under thought, analysis, or any helper keys.
\end{lstlisting}

\noindent Instruction template.\par
\begin{lstlisting}[style=promptblock]
Step 1. Return exactly these top-level keys and no others:
thought, category, modality, summary,
inputs, outputs, tools, external_work, format_strictness, multimodal,
safety_traps, complexity, batch_item_count, batch_unit, repeated_work.

Fill this JSON schema:
{
  "thought": "",
  "category": "",
  "modality": "",
  "summary": "",
  "inputs": [],
  "outputs": [],
  "tools": [],
  "external_work": "none|low|medium|high",
  "format_strictness": "low|medium|high",
  "multimodal": "none|image|audio|video|mixed",
  "safety_traps": [],
  "complexity": "low|medium|high",
  "batch_item_count": 1,
  "batch_unit": "",
  "repeated_work": false
}

If the task requires doing similar work for many files/items/records/pages/messages,
set batch_item_count to the repeated item count and repeated_work to true.

{metadata, local_signal_features, model_trace_features, agent_prompt_section}
\end{lstlisting}

\subsection{Phase-Decomposition Prompt}
\label{app:duration-estimator-step2-prompt}

\noindent System prompt for the local model to decompose the task.\par
\begin{lstlisting}[style=promptblock]
You split task work into phases for timing.
Use only the provided JSON. Do not solve the task.
Return only valid minified JSON. No markdown. No comments.
Begin with { and ensure that all braces are closed.
Use integer seconds.
Keep arrays short: at most 6 items, except for the phases array.
Use plain labels, not long paragraphs.
You may use the 'thought' key to reason before other keys, but do not wrap the other schema keys under thought, analysis, or any helper keys.
\end{lstlisting}

\noindent Instruction template.\par
\begin{lstlisting}[style=promptblock]
Step 2. Create phases.
Return exactly these top-level keys and no others:
thought, category, phases, dominant_risks.

Fill this JSON schema:
{
  "thought": "",
  "category": "",
  "phases": [
    {
      "phase": "",
      "work": "",
      "driver": ""
    }
  ],
  "dominant_risks": []
}

TASK_FEATURES:
{features from Step 1}
\end{lstlisting}

\subsection{Proof of Theorem~\ref{thm:greenclaw-optimal}}
\label{app:greenclaw-optimal-proof}

\begin{proof}
For each $k\in\{1,\ldots,K\}$, region $r\in\mathcal{R}$, and candidate-interval starting time $t\in\mathcal{H}$, define $D^{(k)}_{r,t}$ to be the minimum objective value for the first $k$ ordered execution intervals among all feasible partial placements whose $k$-th execution interval runs in region $r$ during the candidate interval starting at $t$.
If no such partial placement exists, set $D^{(k)}_{r,t}=+\infty$.
Throughout the proof, the minimum over an empty set is defined to be $+\infty$.

We first prove the recurrence by induction on $k$.
For $k=1$, a placement contains one execution interval and has no predecessor, so no cache-recomputation penalty can appear.
The only cost is the predicted carbon-emission term $\widehat{E}^{(1)}_{r,t}\widehat{CI}_{r,t}$, matching the base case in \Cref{sec:theory-dp}.
Now assume the claim holds for $k-1$.
Consider any feasible placement ending execution interval $k$ in region $r$ during the candidate interval starting at $t$.
The previous execution interval must use some region $q$ and some earlier candidate-interval starting time $s<t$.
By the induction hypothesis, the best such predecessor placement ending in region $q$ at time $s$ has value $D^{(k-1)}_{q,s}$.
Under \Cref{assump:local-cache-cost}, the transition from $q$ to $r$ adds no cache-recomputation cost if $q=r$ and adds $\Gamma^{(k)}_{r,t}$ if $q\ne r$.
Therefore, the best feasible predecessor has value
\begin{equation}
    \min_{\substack{q\in\mathcal{R},\ s\in\mathcal{H}:\ s<t}}
        \left\{D^{(k-1)}_{q,s}
        +\mathbf{1}\{q\ne r\}\Gamma^{(k)}_{r,t}\right\}.
\end{equation}
Adding the predicted carbon-emission term $\widehat{E}^{(k)}_{r,t}\widehat{CI}_{r,t}$ gives \Cref{eq:interval-dp-recurrence}.
Thus the induction claim holds for $k$.

The ordered problem formulation fixes the order of the $K$ execution intervals.
After $K$ layers, every feasible placement has exactly one final state $(r,t)$, and every finite final state corresponds to a feasible placement with strictly increasing candidate-interval starting times.
Therefore, minimizing $D^{(K)}_{r,t}$ over all final states gives the global optimum of \Cref{eq:ordered-interval-objective}.
The stored backpointers recover a placement plan attaining this optimum.

It remains to justify the stated complexity.
For a fixed layer $k\ge2$, define the prefix predecessor summaries
\begin{equation}
    P_q(t)=\min_{s\in\mathcal{H}:\ s<t}D^{(k-1)}_{q,s},
    \qquad q\in\mathcal{R}.
\end{equation}
For each region $q$, a left-to-right scan over candidate intervals updates $P_q(t)$ in $O(H)$ time, so all prefix summaries take $O(|\mathcal{R}|H)$ time.
Using these summaries, the predecessor term in \Cref{eq:interval-dp-recurrence} can be written as
\begin{equation}
    \min\left\{
        P_r(t),
        \Gamma^{(k)}_{r,t}+\min_{q\in\mathcal{R}:\ q\ne r}P_q(t)
    \right\}.
\end{equation}
For each candidate-interval starting time $t$, store the two smallest region--value pairs among $\{(q,P_q(t)):q\in\mathcal{R}\}$, with the pairs associated with distinct regions and ordered with multiplicity.
Thus, if two regions attain the same minimum value, both pairs may store that same value.
The quantity $\min_{q\ne r}P_q(t)$ is the first stored value when its associated region is not $r$, and otherwise it is the second stored value.
When $|\mathcal{R}|=1$, we use the convention $\min_{q\ne r}P_q(t)=+\infty$.
Hence every transition value in \Cref{eq:interval-dp-recurrence} can be evaluated in constant time after these summaries are available.
Filling one layer takes $O(|\mathcal{R}|H)$ time, and filling all $K$ layers takes $O(K|\mathcal{R}|H)$ time.
Keeping only the previous and current dynamic-programming layers, together with the prefix summaries used to evaluate the recurrence, requires $O(|\mathcal{R}|H)$ working memory.
If all backpointers are stored to recover the placement plan directly, they require an additional $O(K|\mathcal{R}|H)$ memory.
\end{proof}

\subsection{Proof of Theorem~\ref{thm:interval-deadline-guarantee}}
\label{app:interval-deadline-proof}

\begin{proof}
The returned schedule has finite dynamic-programming value, so it consists of $K$ ordered execution intervals placed in deadline-feasible candidate intervals.
By construction of the recurrence, each finite value in layer $k\ge2$ comes from a finite value in layer $k-1$ with predecessor starting time $s<t$.
Because every selected starting time belongs to
$\mathcal{H}=\{a+j\Delta t\mid j=0,1,\ldots,H-1\}$,
the recovered backpointer chain satisfies
\begin{equation}
    a\le t_1<t_2<\cdots<t_K\le a+(H-1)\Delta t.
\end{equation}
Moreover, distinct elements of $\mathcal{H}$ differ by at least $\Delta t$, so $t_{k+1}\ge t_k+\Delta t$ for every $k<K$.
Consequently, the selected candidate intervals are non-overlapping and appear in chronological order.
The final candidate interval ends at
\begin{equation}
    t_K+\Delta t
    \le a+H\Delta t
    \le d,
\end{equation}
where the last inequality follows from $H=\lfloor(d-a)/\Delta t\rfloor$.
Thus all reserved execution time lies within $[a,d]$.

The number of execution intervals is $K=\lceil\widetilde{T}/\Delta t\rceil$, so the total reserved execution capacity is
\begin{equation}
    K\Delta t
    =\left\lceil \frac{\widetilde{T}}{\Delta t}\right\rceil\Delta t
    \ge \widetilde{T}.
\end{equation}
If $T_{\mathrm{run}}\le\widetilde{T}$, then the selected intervals provide at least the task's actual accumulated active runtime, excluding only idle waiting between execution intervals.
Under the execution-controller model in \Cref{sec:design-execution-controller}, the agent executes during the selected intervals, may be paused at an interval boundary and resumed in a later selected interval, and stops once the task completes.
Therefore, the task can consume all of its required active runtime within the reserved intervals.
Since the final selected interval ends no later than $d$, the task completes by the deadline.
\end{proof}

\end{document}